\documentclass[journal]{IEEEtran}
\usepackage{amsthm,amsfonts}
\usepackage{amsmath}
\usepackage{algorithmic}
\usepackage{algorithm}
\usepackage{array}
\usepackage{textcomp}
\usepackage{stfloats}
\usepackage{url}
\usepackage{subfig}
\usepackage{booktabs}
\usepackage{verbatim}
\usepackage{graphicx}
\usepackage{cite}
\usepackage{bbm}
\usepackage{color}
\usepackage{multirow}
\usepackage{arydshln}
\usepackage[breaklinks=true,bookmarks=false]{hyperref}
\newtheorem{theorem}{Proposition}

\makeatletter
\def\ps@IEEEtitlepagestyle{%
  \def\@oddfoot{\mycopyrightnotice}%
  \def\@evenfoot{}%
}
\def\mycopyrightnotice{%
  {}
}
\makeatother

\begin{document}

\title{A Multi-In and Multi-Out Dendritic Neuron Model and its Optimization}

\author{Yu Ding,
        Jun Yu,
        Chunzhi Gu,
        Shangce Gao,
        and~Chao~Zhang
\thanks{Corresponding Author: Chao Zhang.}
\thanks{Y. Ding, C. Gu and C. Zhang* are with the School of Engineering, University of Fukui, Fukui, Japan (e-mails: dyd22803@u-fukui.ac.jp;gu.univ.work@gmail.com; zhang@u-fukui.ac.jp).}
\thanks{J. Yu is with the Institute of Science and Technology, Niigata University, Niigata, Japan (e-mails: yujun@ie.niigata-u.ac.jp).}
\thanks{S. Gao is with the Faculty of Engineering, University of Toyama, Toyama, Japan (e-mails: gaosc@eng.u-toyama.ac.jp).}
}



\maketitle

\begin{abstract}
Artificial neural networks (ANNs), inspired by the interconnection of real neurons, have achieved unprecedented success in various fields such as computer vision and natural language processing. Recently, a novel mathematical ANN model, known as the dendritic neuron model (DNM), has been proposed to address nonlinear problems by more accurately reflecting the structure of real neurons. However, the single-output design limits its capability to handle multi-output tasks, significantly lowering its applications. In this paper, we propose a novel multi-in and multi-out dendritic neuron model (MODN) to tackle multi-output tasks. Our core idea is to introduce a filtering matrix to the soma layer to adaptively select the desired dendrites to regress each output. Because such a matrix is designed to be learnable, MODN can explore the relationship between each dendrite and output to provide a better solution to downstream tasks. We also model a telodendron layer into MODN to simulate better the real neuron behavior. Importantly, MODN is a more general and unified framework that can be naturally specialized as the DNM by customizing the filtering matrix.
To explore the optimization of MODN, we investigate both heuristic and gradient-based optimizers and introduce a 2-step training method for MODN. Extensive experimental results performed on 11 datasets on both binary and multi-class classification tasks demonstrate the effectiveness of MODN, with respect to accuracy, convergence, and generality. 
\end{abstract}


\section{Introduction}\label{intro}
\IEEEPARstart{N}{eurons} perform essential neurological functions in living organisms such that they can perceive stimuli, integrate information, and conduct impulses. It is composed of somas and synapses, which serve as the structural and functional units of the neural tissue. Specifically, when a neuron receives chemicals released by other neurons, a shift of potential difference will be induced between the inside and outside of the neuron. Once this potential difference exceeds a threshold, the receiving neuron becomes excited and then releases chemicals to communicate with neighboring neurons. This process is known as nerve impulse transmission. In 1943, McCulloch and Pitts conceptualized this biological process and constructed the McCulloch-Pitts neuron, which remains one of the most widely recognized neuron model at present \cite{mcculloch1943logical}. Subsequently, Frank \cite{rosenblatt1961principles} connected McCulloch-Pitts neurons in a fully connected manner in multiple layers. This model, termed the multi-layer perceptron (MLP), is currently the most popular artificial neural network. Despite its conceptual novelty, as pointed out in \cite{marvin1969perceptrons}, this model suffered from two major deficiencies in 1969: (i) a successful training for an MLP is nearly impossible because of its complicated structure; (ii) even if such a multi-layer manner could resolve nonlinear problems, the information processing efficiency of each single-layer perceptron is largely diminished compared to that of the real neuron. 


Later, backpropagation (BP) algorithm was proposed as an effective training method for MLP \cite{rumelhart1986learning} and addressed the above limitation (i). This directly results in broad attention \cite{hornik1989multilayer, gallant1990perceptron, bourlard1990links} to further study of MLP. For example, Hornik et al. \cite{hornik1989multilayer} demonstrated that if the network size is not limited, an MLP with deep structures can approximate any Borel measurable function with arbitrary squashing functions as the activation. Consequently, artificial neural networks (ANNs) have contributed to significant progress in various fields such as image processing \cite{krizhevsky2017imagenet}, natural language processing \cite{vaswani2017attention}, and robotics \cite{he2016adaptive}.

Meanwhile, other efforts have been made to improve the processing strength of single-neuron models. These models attempt to simulate the transmission process in a single neuron \cite{koch2000role, aguera2000can, costa2011one}, such as
pi–sigma models \cite{shin1991pi, ghosh1992efficient, fan2022convergence}, sigma-pi models \cite{gurney1992training, lenze1994make, heywood1995framework, neville2002transformations, weber2007self, long2007lp}, and a single multiplicative neuron model \cite{yadav2007time, zhao2009pso, bas2016robust, kolay2023new}. Most of these models incorporated synapse-like architectures to solve linear and nonlinear problems. Specifically, nonlinearity is modeled as a single multiplication for implementation simplicity. However, the limited architectures do not consider the biological information contained in other key components of the nervous system, such as dendrites. Although later work \cite{legenstein2011branch, bono2017modelling} integrates synapses and dendritic branching in a unified model, handling the exclusive OR (XOR) with single nonlinear branch neurons remains challenging.

To provide the capability for ANNs in modeling the nonlinearity between synaptic inputs and outputs, Todo et al. proposed a dendritic neuron model (DNM) that formulates a single multiplicative neuron model \cite{todo2014unsupervised}. In particular, DNM employs a series of sigmoids to emulate the nonlinear synapse behavior and includes dendrites and the cell body to mimic the architecture of real neurons. Their later study also verified that the DNM can implement the XOR operation \cite{todo2019neurons}. More recently, Gidon et al. \cite{gidon2020dendritic} discovered a calcium-mediated dendritic potential that proves that one single real neuron can tackle XOR, which further evidences the design feasibility of DNM. Hence, DNM has been widely researched to provide better solutions to real-world tasks \cite{ji2016approximate, gao2019dendritic, todo2019neurons, he2021seasonal, xu2021dendritic, luo2021decision, gao2021fully, yu2022improving, ji2022competitive, peng2023extension}.


DNM, however, still suffers from a major issue regarding the structural limitation: since a single DNM is designed to have one single output, it cannot handle multi-output tasks. This significantly narrows the applications as the majority of real-world tasks involve multiple outputs. Two main approaches can be easily achieved to extend ANNs to resolve this problem. First, one can introduce an additional layer (e.g., sigma-pi \cite{gurney1992training}) to allow multiple outputs. However, this can cause the learning procedure to involve undesired information, because each output is regressed from the entire set of inputs. Alternatively, stacking multiple artificial neurons can yield multiple outputs. However, this completely ignores the inherent correlation among these outputs. Moreover, it should be noted that one single biological neuron branches to other cells at the end of the axon to realize ``multiple outputs'', which differs greatly from the stacking approach. We thus argue that enabling DNM to perform multi-out tasks requires an elaborate modeling strategy that jointly learns the relationship among all outputs and shows awareness of real nervous architectures.

In this paper, we propose a multi-output dendritic neuron model (MODN), which is a novel framework for handling multi-output tasks. Our key idea is to introduce a selection strategy such that MODN learns to adaptively select dendrites to regress each output. To achieve this, we propose a learnable Boolean filtering matrix and integrate it into the soma layer to characterize the preference tendency between each dendrite and output. MODN is a unified framework that can be readily specialized as the DNM by specifying the filtering matrix. We also design a telodendron layer to better model neuron firing to allow a more faithful representation of real neuron architectures. To address the optimization of the MODN, we examine gradient-based and heuristic algorithms, and propose a two-step training method to learn the filtering matrix. We perform extensive experiments on 11 datasets regarding both binary and multi-class classification tasks to evaluate the effectiveness of MODN. Furthermore, we investigate two instances of MODN, which can be treated as DNM extensions, and demonstrate the superiority of MODN in handling multi-output classification tasks. Our contributions can be summarized as follows:

\begin{itemize}
	\item We propose a novel framework that enables  dendritic neuron models to handle both binary and multi-output tasks in a unified mathematical formulation. 
 	\item We introduce a filtering matrix to adaptively select the dendrites for the downstream tasks.
   	\item We experimentally investigate the optimization strategies for our model with gradient-based and heuristic algorithms, and propose a two-step learning method for MODN.  
        \item We report extensive experimental results against prior dendritic neuron model on classification tasks using 11 datasets to evaluate the effectiveness of our model.
\end{itemize}

The remainder of this paper is organized as follows. We review the architecture of classic DNM in Sec. \ref{relate}. We then describe MODN in detail as well as the optimization for it in Sec. \ref{method}. Next, we report experimental results against DNM in Sec. \ref{experiments}. Sec. \ref{con} concludes our paper. 



\section{Preliminaries}
\label{relate}
In this section, we give an introduction to the fundamental dendritic neuron model (DNM) \cite{todo2014unsupervised, gao2019dendritic} since it is the most related work to ours. DNM is a type of artificial neuron model proposed to emulate the neural transmission process in real neurons. It is constituted by four layers: a synaptic layer, a dendrite layer, a membrane layer, and a soma layer. The synaptic and soma layers employ nonlinear activation functions to construct single-input and single-output nodes, respectively. Therefore, DNM is only capable of handling single-out tasks. The multiplicative and additive rules are applied to the dendrite and membrane layers to model the AND and OR functions, respectively.

\noindent \textbf{Synaptic Layer.} The synaptic layer is the initial node for the transmission process in a neuron that directly receives the input signal. It consists of several synapses to simultaneously tackle multiple inputs. A synapse includes presynaptic membrane, synaptic cleft, and postsynaptic membrane. In the nervous system, the presynaptic membrane in the previous neuron emits a chemical, called a neurotransmitter, to the postsynaptic membrane in the next neuron to form an electrical impulse for information transformation. Typically, the electrical impulse can either inhibit or excite the state of the next neuron, depending on the type of neurotransmitter. DNM leverages a sigmoid function to mathematically model such a physiological process. Assume the sample set to be $\boldsymbol{X}=[X^{1},...,X^{n},...,X^{N}]\in{R^{D\times{N}}}$ with $D$ dimensions. Assume further that the $n$-th arbitrary sample is $X^{n}=[x_{1}^{n},...,x_{d}^{n},...,x_{D}^{n}]^{T}$. Then, the nonlinear mapping of the synaptic layer to the $n$-th sample (the superscript $n$ is omitted for better readability in all the following equations and figures) is given by
\begin{equation}
\label{eq:eq1}
Y_{ij}=\frac{1}{1+e^{-\alpha^{s} (\omega^{s}_{ij}x_{i}-\theta^{s}_{ij})}},
\end{equation}
where $x_{i}$ is the $i$-th $(i=1,...,D)$ input to the dendrite, and $Y_{ij}$ is the output of the $i$-th synaptic node on the $j$-th $(j=1,...,M)$ dendrite. The learnable parameter pair ($\omega^{s}_{ij}$, $\theta^{s}_{ij}$), which respectively denotes the weights and thresholds, jointly controls the synaptic states in connecting two neurons. $\alpha^{s}$, termed a synaptic scaling factor, is a positive hyperparameter that indicates the scaling strength. Eq. \ref{eq:eq1} computes the output of $i$-th synapse on the $j$-th dendrite. Synaptic nodes on dendrites respond accordingly to the input signal through these learnable parameters. Specifically, synaptic states can be divided into four types, including the \textit{excitatory state}, the \textit{inhibitory state}, the \textit{unvarying 0 state}, and the \textit{unvarying 1 state}.

\begin{itemize}
        \item \textit{Excitatory state.} The input signal has a positive correlation with the output signal at this state node.
        This state is activated when both $\omega^{s}_{ij}$ and $\theta^{s}_{ij}$ are greater than 0 and $\omega^{s}_{ij}$ is greater than $\theta^{s}_{ij}$.
        \item \textit{Inhibitory state.} The input signal has an inverse correlation with the output signal at this state node.
        This state is activated when both $\omega^{s}_{ij}$ and $\theta^{s}_{ij}$ are smaller than 0 and $\omega^{s}_{ij}$ is smaller than $\theta^{s}_{ij}$.
        \item \textit{Unvarying 0 state.} In this synaptic state, the output is 0 for any input between 0 and 1. This state is activated when both $\omega^{s}_{ij}$ and $\theta^{s}_{ij}$ are greater than 0 and $\theta^{s}_{ij}$ is greater than $\omega^{s}_{ij}$, or $\theta^{s}_{ij}$ is greater than 0 but $\omega^{s}_{ij}$ is smaller than 0.
        \item \textit{Unvarying 1 state.} In this synaptic state, the output is 1 for any input between 0 and 1. This state is activated when both $\omega^{s}_{ij}$ and $\theta^{s}_{ij}$ are smaller than 0 and $\theta^{s}_{ij}$ is smaller than $\omega^{s}_{ij}$, or $\theta^{s}_{ij}$ is smaller than 0 but $\omega^{s}_{ij}$ is greater than 0.
\end{itemize}

\noindent \textbf{Dendrite Layer.} The dendrite layer integrates excitatory and inhibitory outputs from the synaptic layer. 
In modeling the dendrite layer, formulating the complex dendritic structure (e.g., rich ion channels) is mathematically intractable. DNM deals with this issue by leveraging a simple multiplicative process to implement the AND function among synapses for non-linearity, which is given by 
\begin{equation}
\label{eq:eq2}
Z_{j}=\prod_{i=1}^{D}Y_{ij},
\end{equation}
where $Z_{j}$ refers to the output of the $j$-th dendrite.


\noindent \textbf{Membrane Layer.} The membrane layer fuses the entire outputs of the dendrite layer. DNM utilizes the OR operation to fulfill this process, following 
\begin{equation}
\label{eq:eq3}
V=\sum_{j=1}^{M}Z_{j},
\end{equation}
where $V$ represents the output of the membrane layer.

\noindent \textbf{Soma Layer.} Receiving the output of the membrane layer, the soma layer performs neuron firing to form the final output. This process is modeled in DNM as a sigmoid activation, which is formulated as 
\begin{equation}
\label{eq:eq4}
O=\frac{1}{1+e^{-\alpha^{o}(v-\theta^{o})}}, 
\end{equation}
where $O$ represents the final output of the neuron. Similar to the synaptic layer, $\alpha^{o}$ and $\theta^{o}$ denote the two positive scaling hyperparameters that control the activation strength, respectively.

In general, DNM is inspired more by the mechanism of biological neuron cells and 
proves efficient in handling nonlinearity of versatile downstream tasks by modifying the training objective \cite{todo2014unsupervised, gao2019dendritic}. However, it suffers from a major limitation that the soma layer in DNM does not enable multiple outputs, which significantly limits its application. Furthermore, a straightforward extension, such as stacking multiple DNMs, to allow multi-out capability would always confine each dendrite to explore a single input attribute. This would also induce the loss of the essential nature of being a single neuron. We next present MODN to realize effective multi-out learning with a single neuron model.




\section{Method}\label{method}
Let us now detail our method $-$ multi-output dendritic neuron model (MODN), which is proposed to specifically handle tasks with multiple outputs. As illustrated in Fig. \ref{fig:1}, MODN consists of four layers: 1) a synaptic layer; 2) dendrite layer; 3) soma layer; 4) and telodendron layer. Given a $D$-dimensional input sample $x^n = [x_1, \cdots, x_D]$ on each dendrite (e.g., $m$-th dendrite), the synaptic layer first performs dimensional-wise nonlinear transformation to project it into a feature vector $Y_m = [Y_{1m}, \cdots, Y_{Dm}]$. This transformation is performed on all $M$ dendrites to obtain the feature vector set $\mathbf{Y} = [Y_1, \cdots, Y_M]$. For each $Y_m$, the dendrite layer conducts element-wise multiplication for feature fusion. We adopt the design of synaptic and dendrite layers as in \cite{todo2014unsupervised, gao2019dendritic} for MODN. The fused feature set $\mathbf{Z} = [Z_1, \cdots, Z_M]$, which is obtained from all dendrites, is then fed into the soma layer for further feature composition. In designing the soma layer, we propose a novel filtering strategy by introducing a preference matrix to enable an adaptive learning of specific features among the samples. Such features then pass through the telodendron layer, which is the first time modeled in DNMs, to eventually regress the output. MODN is a unified framework for DNMs, naturally including prior DNM \cite{gao2019dendritic} as a special case by customizing the preference matrix. 

\subsection{Multi-output dendritic neuron model}

\begin{figure*}[htb]
\begin{center}
\includegraphics[width=1.0\linewidth]{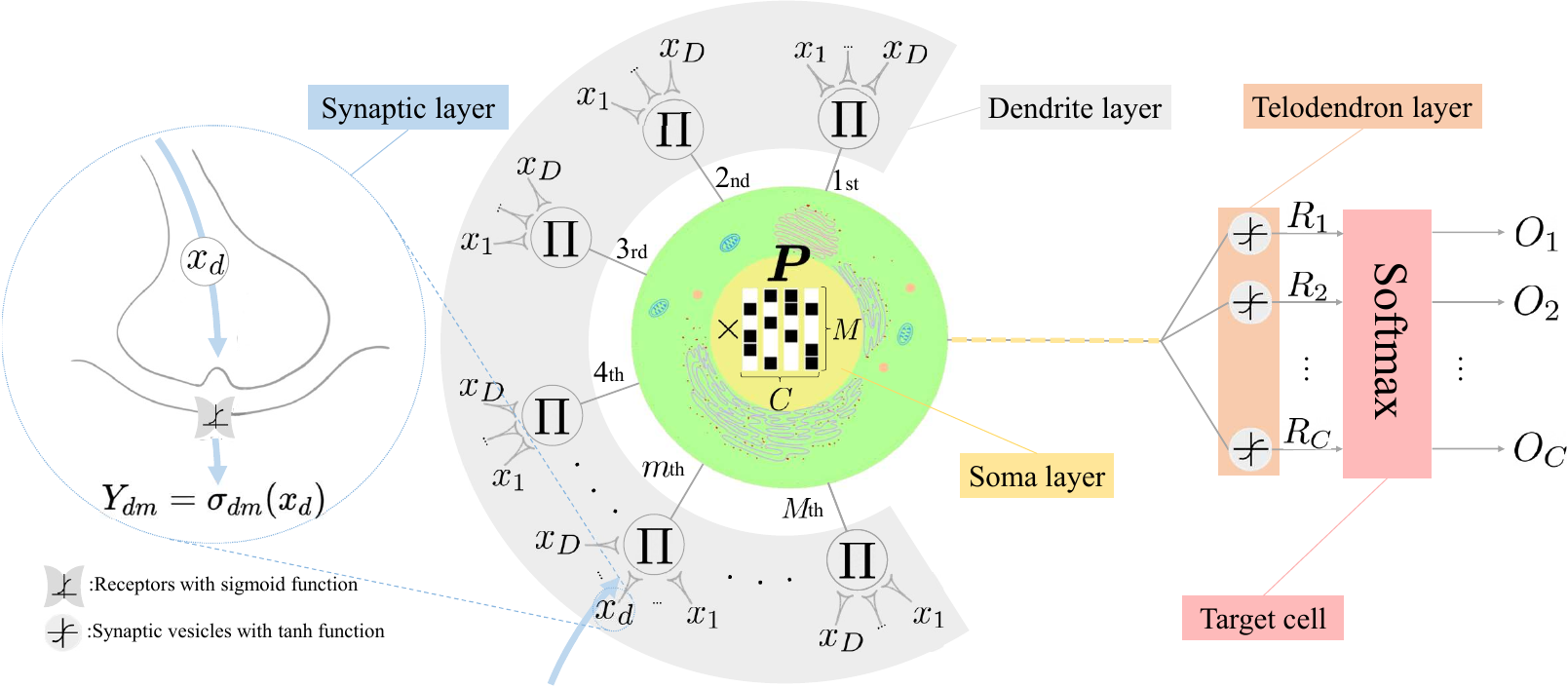}
\end{center}
\caption{Overview of the proposed multi-output dendritic neuron (MODN). The input sample $X$ sequentially passes through the synaptic layer, dendrite layer, soma layer, and telodendron layer to connect to the target cell for the final output $O$. $\sigma(\cdot)$ denotes Eq. \ref{eq:eq1}. The downstream task is assumed to be the multi-class classification as an example.}
\label{fig:1}
\end{figure*}

\noindent \textbf{Preference filter.} In \cite{todo2014unsupervised}, only the addition operation with all outputs of the dendrite layer is performed to form the soma layer feature. However, uniformly treating all the dendrite layer outputs in a naive fashion would, by contrast, hinder the network (i.e., the neuron) from effectively learning the key features required for some specific outputs. This can deprive the ability of the model to select beneficial dendrites when the input attribute varies. We thus argue that the model should be capable of adaptively choosing the desired dendrites to handle the inherent attribute difference to promote better feature selection. To this end, 
we propose a filtering strategy by introducing a Boolean matrix $\boldsymbol{P}\in{\{0,1\}}^{M\times C}$ to selectively learn the key features. Given $Z=[Z_{1},...,Z_{m},...,Z_{M}]$ as the input to the soma layer (i.e., the output of the dendritic layer), the filtering is performed by 
\begin{equation}
\label{eq:eq5}
V=Z{\boldsymbol{P}}
\end{equation}
to output the soma layer feature $V=[V_{1},...,V_{c},...,V_{C}]\in \mathbbm{R}^{1\times C}$. In our design, we make $\boldsymbol{P}$ learnable to adaptively learn the beneficial features for downstream tasks. Particularly, the rows and columns of $\boldsymbol{P}$ correspond to all the dendrites and soma layer outputs. By learning $\boldsymbol{P}$, each column of $\boldsymbol{P}$ is endowed with the role in guiding the dendrites to produce the fused features for one specific soma layer output. As such, the final combined features are disentangled to mitigate learning overlapped feature on each dendrite. 

\noindent \textbf{Dendrites state.} The introduction of the filtering matrix $\boldsymbol{P}$ during learning directly causes the arbitrary dendrite to exhibit different states. In particular, the following three dendrite states can be discussed: 
\begin{itemize}
        \item \textit{Inoperative state.} This state is activated when all the elements of in a row of $\boldsymbol{P}$ are equal to 0. Here, the associated dendrite does not influence the neuron since the signal generated by the dendrite corresponding to the column does not propagate forward. 
        \item \textit{Exclusive state.} This state is activated when one of the rows in the matrix $\boldsymbol{P}$ is an one-hot vector. It indicates that the dendrite in connection with this row only captures the exclusive attributes corresponding to the output of this column for downstream tasks.
        \item \textit{Communal state.} This state is activated when a row in the matrix $\boldsymbol{P}$ contains more than one element of 1. Here, the dendrite corresponding to this row is able to learn the joint features for the outputs with 1-valued elements in $\boldsymbol{P}$ within the same row.
\end{itemize}
Among these states, the \textit{Exclusive state} and the \textit{Communal state} indicate that model learning is progressing. However, in the case of the \textit{Inoperative state}, since the inoperative dendrites would trigger the halt in the neural transmission process, including this state of dendrite does not have any positive effect on the downstream task (e.g., classification). Instead, it induces increasing computational overhead. We thus empirically tune an appropriate dendrite number $M$ to hopefully exclude the inoperative dendrite. Consequently, we only include the \textit{Exclusive state} and the \textit{Communal state} in our model for encoding efficiency.

\noindent \textbf{Two instances of MODN.}
Because MODN is heavily dependent on $\boldsymbol{P}$, modifying the elements in $\boldsymbol{P}$ results in diverse modes of MODN. In addition to our learnable design, manual specification of $\boldsymbol{P}$ can also be conducted in advance, which instantiates the model to two special cases. To show the superiority of automatic optimization of $\boldsymbol{P}$, these two instances will also be compared in the experiment (Sec. \ref{evaluation}).

\begin{itemize}
        \item \textit{MODNP.} The MODN can be customized to learn no correlations between outputs. As illustrated in Fig. \ref{fig:2}(a), this example degrades MODN to the case of stacking multiple single dendritic neural models (i.e., DNMs) by discarding such output correlations. In this case, the model is equivalent to Fig. \ref{fig:2}(b) because each dendrite group (i.e., gray, blue, and orange DNMs) works independently during learning. To achieve this, $\boldsymbol{P}$ can be designed as 
        \begin{equation}
        \label{eq:eq6}
        P_{ij}=\mathbbm{1}(\lceil i(C/M) \rceil=j), C \; \text{mod} \; M = 0,
        \end{equation}
        where $P_{ij}$ denotes each element within $\boldsymbol{P}$. $\mathbbm{1}(\cdot)$ and $\lceil \cdot \rceil$ refer to the indicator and the ceiling function, respectively.
        As such, each independent group of dendrites in MODNP can be regarded as the dendrites of a single DNM model since they are portioned into groups. We thus regard this case as an MODN instance, and term it as MODN partition (MODNP).
        
        \item \textit{MODNF.} In addition, MODN can be configured to show complete awareness of the relationships between all dendrites and all outputs. This can be realized by simply setting  $\boldsymbol{P}$ as an all-ones matrix, i.e., 
        \begin{equation}
        \label{eq:eq7}
        P_{ij}=1, i=1,\cdots,M, j=1,\cdots,C
        \end{equation}
        As such, the attributes relevant to all outputs are allowed to be jointly characterized by each dendrite. We regard this case as another MODN instance, and term it as MODN-full (MODNF). It can be viewed further as a special case of \textit{Communal state}. 
\end{itemize}
Since MODN is aimed at addressing multi-output tasks, the above two instances represent the two most straightforward extensions that one would naively and intuitively devise from DNM. However, even if MODN naturally includes these two instances, enforcing them by constraining $\boldsymbol{P}$ can hinder the performance because they involve different learning constraints by design. By contrast, MODN can by nature explore a better effectiveness to grasp the intrinsic attributes of individual output, and the mutual dependencies among the outputs. As will be shown in our experiments (Sec. \ref{experiments}), learning such a preference filter, instead of pre-determining it, would generally contribute to a performance boost.

\begin{figure}[]
\centering
  \includegraphics[width=1\linewidth]{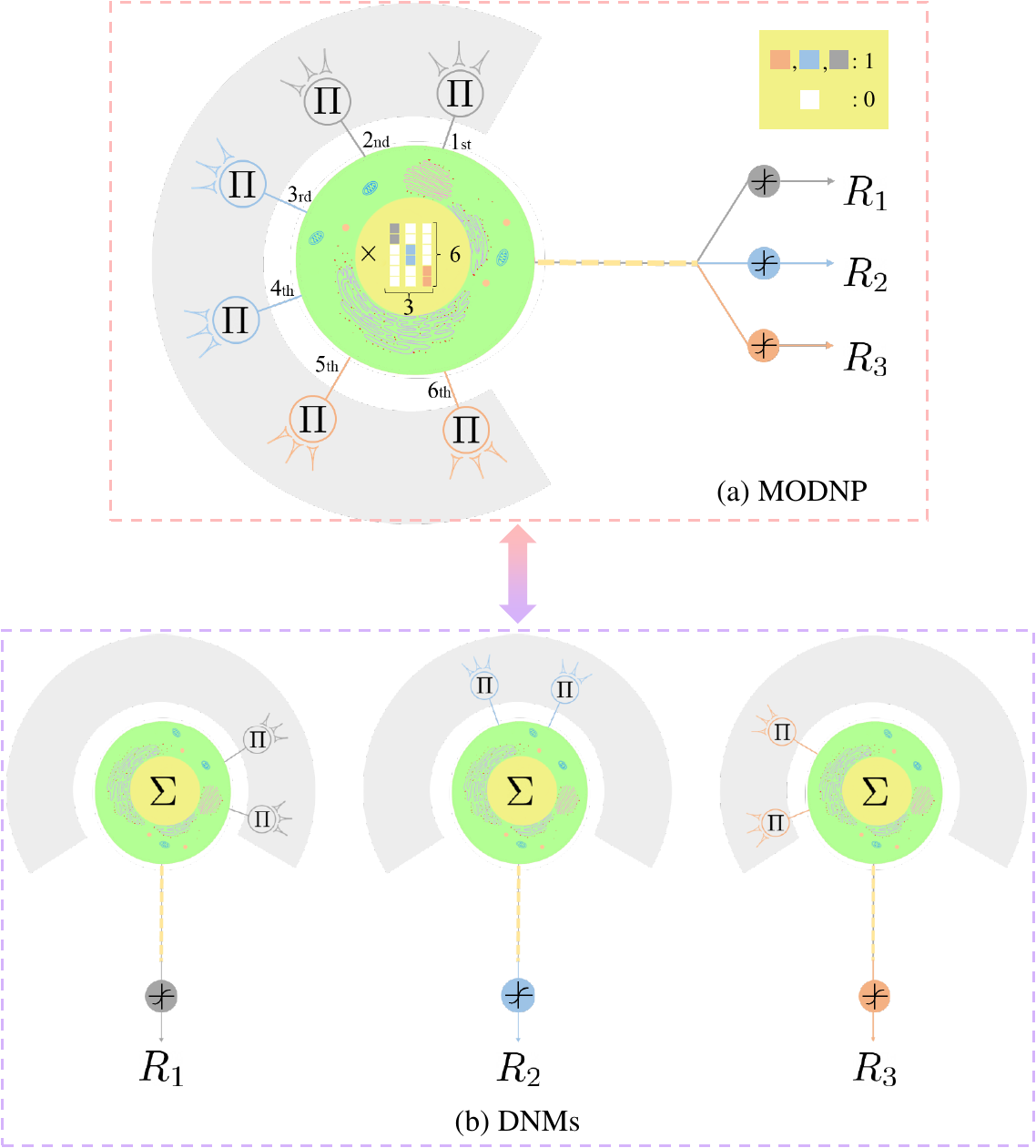}
   \caption{Illustration of equivalence between (a) MODNP and (b) Stacked DNMs. The number of dendrites $M$ and output dimension $C$ are set to 6 and 3 as an example, respectively. Dendrites colored the same belong to the same dendrite group.} 
\label{fig:2}
\end{figure}



\noindent \textbf{Telodendron layer.} In the nervous systems, axons, which are the bridges of multiple neurons (or target cells), generally branch out to form telodendria to stimulate neuron firing. Motivated by this, we introduce a telodendron layer prior to the output regression to encourage the nonlinearity of MODN. Specifically, our telodendron layer is structured with a tanh activation to map the soma layer output $V$ to the range [-1, 1], following
\begin{equation}
\label{eq:eq8}
R_{c}=\frac{e^{\alpha^{t}(\omega^{t}_{c}V_{c}-\theta^{t}_{c})}-e^{-\alpha^{t}(\omega^{t}_{c}V_{c}-\theta^{t}_{c})}}{e^{\alpha^{t}(\omega^{t}_{c}V_{c}-\theta^{t}_{c})}+e^{-\alpha^{t}(\omega^{t}_{c}V_{c}-\theta^{t}_{c})}},
\end{equation}
where $R_c$ $(c=1,...,C)$ refers to the $c$-th telodendron output. $\alpha^{t}$ denotes the telodendron scaling factor, which is similar to that of the synaptic layer. To enable the flexibility of MODN in controlling the activation (i.e., neuron fire), we introduce learnable weights $\omega^{t}_{c}$ and thresholds $\theta^{t}_{c}$ to from the $c$-th telodendron. Note that MODN differs from DNM in its activation design regarding: (1) the additional telodendron layer introduced in MODN is more biological-oriented rather than simply stacking multiple DNMs in handling the multi-output tasks; (2) DNM simply relies on a fixed sigmoid activation, which would induce huge labor efforts in fine-tuning the optimal thresholds. 

\noindent \textbf{Target cell.} 
The telodendron layer output would pass into next cells to enable performing biological roles, which corresponds to different designs of the target cell. In MODN, different designs of the target cell can be tailored to different downstream tasks. For example, in a multi-class classification task, the target cell can be formed with a softmax operation for normalizing the output. Thus, our final model output to derive the classification results can be expressed as  
\begin{equation}
\label{eq:eq9}
O_{c}=\frac{e^{R_{c}}}{\sum_{c=1}^{C}e^{R_{c}}},
\end{equation}
where $O_{c}$ is the probability that the input sample belongs to the $c$-th class. Given a training set with $N$ samples, the supervision can be enforced by employing the (binary) cross-entropy, following 
\begin{equation}
\label{eq:eq10}
\mathcal{L}=-\frac{1}{N}\sum_{n=1}^{N}\ell^{n}=
-\frac{1}{N}\sum_{n=1}^{N}\sum_{c=1}^{C}\widehat{O}_{c}^{n}\log(O_{c}^{n}),
\end{equation}
where $\widehat{\cdot}$ is the corresponding ground truth. $\ell^{n}$ is the loss of $n$-th sample. 

Up to this point, we have detailed our proposed multi-out dendritic neuron model and its two instances. We next need to determine the optimization algorithm for learning.

\subsection{Optimization}
A prevalent selection for optimizing the network such as deep neural network is BP. However, the fact that $\boldsymbol{P}$ is discrete in its solution space makes our model less feasible to directly take partial derivative and propagate the gradient. Hence, learning MODN requires other gradient-free techniques. To address this, we resort to heuristic algorithms to explore decent solutions within acceptable processing time. Particularly, we employ five types of heuristic algorithms to examine the performance: biogeography-based optimization (BBO) \cite{simon2008biogeography}, genetic algorithm (GA) \cite{simon2013evolutionary}, particle swarm optimization (PSO) \cite{eberhart1995new}, population-based incremental learning (PBIL) \cite{baluja1994population}, and evolutionary strategy (ES) \cite{engelbrecht2007computational}. 

To be more specific, training MODN involves learning the following two components: 
(a) the model parameters $\omega^{s}_{ij}, \theta^{s}_{ij}, \omega^{t}_{c}$, $\theta^{t}_{c}$, and (b) the preference filter $\boldsymbol{P}$. During training, we notice that jointly learning (a) as well as (b) would complicate the optimization for both parts. The reason is that, since the optimization for  $\boldsymbol{P}$ would cause its elements to flip (i.e., 1 to 0 or 1 to 0), the multiplication operation (Eq. \ref{eq:eq5}) of $\boldsymbol{P}$ and the signals can activate and deactivate the learned parameters as the training progresses. To alleviate the resulting negative influence and ease training, we follow a 2-step optimization scheme by first learning (b) yet freezing (a), and then updating (a) yet freezing (b) alternately. 



For MODNP and MODNF, we simply utilize BP since $\boldsymbol{P}$ is pre-determined and does not require any update. By recursively using the results in the previous round for $R$ iterations, the $(r+1)$-th $(r=1,...,R)$ round of update for $\omega^{s}_{ij}, \theta^{s}_{ij}, \omega^{t}_{c}, \theta^{t}_{c}$ can be formulated as: 
\begin{equation}
\label{eq:eq11}
w^{s}_{ij(r+1)}=w^{s}_{ij(r)}+\eta  \frac{1}{N}\sum^{N}_{n=1}\frac{\partial\ell^{n}}{\partial w^{s}_{ij(r)}},
\end{equation}
\begin{equation}
\label{eq:eq12}
\theta^{s}_{ij(r+1)}=\theta^{s}_{ij(r)}+\eta  \frac{1}{N}\sum^{N}_{n=1}\frac{\partial\ell^{n}}{\partial \theta^{s}_{ij(r)}},
\end{equation}
\begin{equation}
\label{eq:eq13}
w^{t}_{c(r+1)}=w^{t}_{c(r)}+\eta  \frac{1}{N}\sum^{N}_{n=1}\frac{\partial\ell^{n}}{\partial w^{t}_{c(r)}},
\end{equation}
\begin{equation}
\label{eq:eq14}
\theta^{t}_{c(r+1)}=\theta^{t}_{c(r)}+\eta  \frac{1}{N}\sum^{N}_{n=1}\frac{\partial\ell^{n}}{\partial \theta^{t}_{c(r)}},
\end{equation}
where $\eta$ denotes the learning rate. In particular, the partial derivatives of $\ell^{n}$ with respect to $w^{t}_{c}$ and $\theta^{t}_{c}$ in Eqs. \ref{eq:eq13} and \ref{eq:eq14} (the superscript $n$ is omitted for brevity in all the following equations) can be respectively derived as:
\begin{equation}
\begin{aligned}
\label{eq:eq15}
\frac{\partial\ell}{\partial w^{t}_{c(r)}}&= \frac{\partial\ell}{\partial O_c}\frac{\partial O_c}{\partial R_c}\frac{\partial R_c}{\partial w^{t}_{c}}\\
&=\widehat{O}_{c}\alpha^{t}V_{c}(1-O_{c})(1-R_c^2),
\end{aligned}
\end{equation}
\begin{equation}
\begin{aligned}
\label{eq:eq16}
\frac{\partial\ell}{\partial \theta^{t}_{c(r)}}&= \frac{\partial\ell}{\partial O_c}\frac{\partial O_c}{\partial R_c}\frac{\partial R_c}{\partial\theta^{t}_{c}}\\
&=\widehat{O}_{c}\alpha^{t}(O_{c}-1)(1-R_c^2),
\end{aligned}
\end{equation}
where $c$ refers to the 1-valued index in the corresponding ground-truth label represented by an one-hot encoding vector.

Since a different design of $\boldsymbol{P}$ results in varying filtering strategies, the update for MODNP and MODNF needs to be separately considered. For MODNP, the gradients with respect to $w^{s}_{ij}$ and $\theta^{s}_{ij}$ can be calculated as
\begin{equation}
\begin{aligned}
\label{eq:eq17}
\frac{\partial\ell}{\partial w^{s}_{ij(r)}}= &\frac{\partial\ell}{\partial O_c}\frac{\partial O_c}{\partial R_c}\frac{\partial R_c}{\partial V_c}\frac{\partial V_c}{\partial Z_{j}}\frac{\partial Z_{j}}{\partial Y_{ij}}\frac{\partial Y_{ij}}{\partial w^{s}_{ij}}\\
=&\widehat{O}_{c}\alpha^{t}w^{t}_{c}\alpha^{s}x_{i}(1-O_{c})(1-R_c^2)p_{jc}\\
&(\prod_{q=1(q \neq i)}^{D}Y_{qj})Y_{ij}(Y_{ij}-1),
\end{aligned}
\end{equation}

\begin{equation}
\begin{aligned}
\label{eq:eq18}
\frac{\partial\ell}{\partial \theta^{s}_{ij(r)}}=&\frac{\partial\ell}{\partial O_c}\frac{\partial O_c}{\partial R_c}\frac{\partial R_c}{\partial V_c}\frac{\partial V}{\partial Z_{j}}\frac{\partial Z_{j}}{\partial Y_{ij}}\frac{\partial Y_{ij}}{\partial\theta^{s}_{ij}}\\
=&\widehat{O}_{c}\alpha^{t}w^{t}_{c}\alpha^{s}(1-O_{c})(1-R_c^2)p_{jc}\\
&(\prod_{q=1(q \neq i)}^{D}Y_{qj})Y_{ij}(1-Y_{ij}),
\end{aligned}
\end{equation}
where $p_{jc}$ is the element of the $j$-the row and the $c$-the column of the $\boldsymbol{P}$; for MODNF, the gradients with respect to $w^{s}_{ij}$ and $\theta^{s}_{ij}$ can be measured via
\begin{equation}
\begin{aligned}
\label{eq:eq19}
\frac{\partial\ell}{\partial w^{s}_{ij(r)}}= &\frac{\partial\ell}{\partial O_c}O_{c}(\frac{\partial R_c}{\partial w^{s}_{ij}}-\sum_{s=1}^{C}O_{s}\frac{\partial R_s}{\partial w^{s}_{ij}})\\
=&\widehat{O}_{c}(\frac{\partial R_c}{\partial w^{s}_{ij}}-\sum_{s=1}^{C}O_{s}\frac{\partial R_s}{\partial w^{s}_{ij}}),
\end{aligned}
\end{equation}
\begin{equation}
\begin{aligned}
\label{eq:eq20}
\frac{\partial\ell}{\partial \theta^{s}_{ij(r)}}= &\frac{\partial\ell}{\partial O_c}O_{c}(\frac{\partial R_c}{\partial \theta^{s}_{ij}}-\sum_{s=1}^{C}O_{s}\frac{\partial R_s}{\partial \theta^{s}_{ij}})\\
=&\widehat{O}_{c}(\frac{\partial R_c}{\partial \theta^{s}_{ij}}-\sum_{s=1}^{C}O_{s}\frac{\partial R_s}{\partial \theta^{s}_{ij}}),
\end{aligned}
\end{equation}

\begin{equation}
\begin{aligned}
\label{eq:eq21}
\frac{\partial\mathcal{R}_{s}}{\partial w^{s}_{ij}} = &\frac{\partial R_{s}}{\partial V_{s}}\frac{\partial V_{s}}{\partial Z_{j}}\frac{\partial Z_{j}}{\partial Y_{ij}}\frac{\partial Y_{ij}}{\partial w^{s}_{ij}}(s = 1, ..., C)\\
=&\alpha^{t}w^{t}_{s}\alpha^{s}x_{i}(1-R_s^2)\\
&(\prod_{q=1(q \neq i)}^{D}Y_{qj})Y_{ij}(1-Y_{ij}),
\end{aligned}
\end{equation}

\begin{equation}
\begin{aligned}
\label{eq:eq22}
\frac{\partial\mathcal{R}_{s}}{\partial \theta^{s}_{ij}} = &\frac{\partial R_{s}}{\partial V_{s}}\frac{\partial V_{s}}{\partial Z_{j}}\frac{\partial Z_{j}}{\partial Y_{ij}}\frac{\partial Y_{ij}}{\partial \theta^{s}_{ij}}(s = 1, ..., C)\\
=&\alpha^{t}w^{t}_{s}\alpha^{s}(1-R_s^2)\\
&(\prod_{q=1(q \neq i)}^{D}Y_{qj})Y_{ij}(Y_{ij}-1).
\end{aligned}
\end{equation}
Note that we modify all zero-valued $Y$ to 1 during training to prevent the phenomenon of ``dendrite dying'' \footnote{We treat $Y$ as 0 when it is smaller than $10^{-6}$ in all experiments.}. That is, when there exists one connection condition of the synaptic node which falls into the \textit{Unvarying 0 state}, the gradient in the corresponding dendrite will always be reduced to 0 and never receive updates. A detailed derivation is given in \textbf{Proposition} \ref{T1}.

\begin{theorem}\label{T1}
   For any given integer $j$, if any $Y_{ij}$ in Eq. \ref{eq:eq2} is non-zero, the gradient of the $j$-th dendrite is a non-zero vector.
\end{theorem}

\begin{proof}
    \textit{The proof for} \textbf{Proposition} \ref{T1} \textit{is as follows. Here, by referring to Eq. \ref{eq:eq2}, we assume the following multiplication:}     
    \begin{equation}
    \label{eq:eq23}
        z=\prod_{i=1}^{D}y_{i},
    \end{equation}
    \textit{where we omit the subscript $j$ for simplicity without the loss of generality. With a minor abuse of notation, we define the variable set $Y=\{y_1,\cdots,y_D\}$. By taking the partial derivative for Eq. \ref{eq:eq23} with respect to $y_p\in Y$, we obtain}     
    \begin{equation}
    \label{eq:eq24}
        \frac{\partial z}{\partial y_p}=\prod_{i=1,  i\neq p}^{D}y_{i}.
    \end{equation}
    \textit{Given the number of zero-valued variables in $Y$, the following three cases can be discussed during the model update:}
    \begin{itemize}
        \item     \textit{CASE 1: If $\forall y_p\in Y, y_{p}\neq 0$ holds, then the partial derivatives obtained via Eq. \ref{eq:eq24} is a non-zero value. Therefore, the network can be successfully updated.}
        \item   \textit{CASE 2: If at least two variables in $Y$ are 0 (e.g., $y_{p}$ and $y_{q}$), then the partial derivatives with respect to all variables are always 0 despite the update being performed, because the partial derivatives (Eq. \ref{eq:eq24}) should always include one or both of $y_{p}$ and $y_{q}$. Therefore, the issue of ``dendrite dying'' is triggered and the network fails to update. }        
        \item   \textit{CASE 3: If one and only one variable (e.g., $y_{p}$) is 0, then the partial derivatives with respect to the remaining variables in Y (i.e., $\{y_1,\cdots,y_{p-1},y_{p+1}, \cdots, y_D\}$) are also 0. Ideally, because the partial derivative with respect to $y_{p}$ is still a non-zero value, all of variables will be non-zero values in the next round of update. However, according to the chain rule, the update calculation still involves multiplying the partial derivative of Eq. \ref{eq:eq1}, which contains the $y_{p}$. As such, the gradients of $\ell$ in Eq. \ref{eq:eq10} with respect to $w^{s}_{ij(r)}$ and $\theta^{s}_{ij(r)}$ always remain 0. Therefore, the network is still unable to successfully update since CASE 3 is consequently equivalent to CASE 2.}
        
    \end{itemize}
\end{proof}

To nonetheless tackle this issue, we simply set $Y_{qj}$ to 1 to enable back-propagation without manually imposing additional constraints. We also adopt the heuristic algorithms to optimize the MODNP and MODNF (see Sec. \ref{experiments}).


\section{Experiments}\label{experiments}

\begin{table*}[htb]
\centering
\caption{Detailed descriptions of the classification datasets.}
\label{table1}
\begin{tabular}{cccccc}
\toprule
\textbf{Dataset}              & \textbf{Num. of training samples} & \textbf{Num. of testing samples} & \textbf{Num. of classes} & \textbf{Dimensionality} & \textbf{Description}                 \\ \hline
Breast cancer        & 546                      & 137                     & 2               & 9              & benign or malignant         \\
Blood transfusion    & 598                      & 150                     & 2               & 4              & donate blood or not         \\
Heart disease        & 212                      & 91                      & 2               & 13             & angiographic disease status \\
Raisin grains        & 720                      & 180                     & 2               & 7              & varieties of raisin         \\
Caesarian Section    & 64                       & 16                      & 2               & 5              & recommend cesarean or not   \\
Glass identification & 171                      & 43                      & 6               & 9              & types of glasses            \\
Wine                 & 142                      & 36                      & 3               & 13             & types of wines              \\
Car evaluation       & 1209                     & 519                     & 4               & 6              & car acceptability           \\
Iris                 & 90                       & 60                      & 3               & 4              & types of iris               \\
Seeds                & 168                      & 42                      & 3               & 7              & varieties of wheat          \\
Ecoli                & 228                      & 99                      & 5               & 7              & protein localization sites  \\ \bottomrule
\end{tabular}
\end{table*}
In this section, we present extensive experimental results on 11 benchmark datasets against other ANNs to evaluate MODN. \textit{All quantitative evaluations in Tab. \ref{table4} and \ref{table5} in the following of this section report the averaged experimental results over 30 runs to respect the criterion of significance test.}

\subsection{Datasets and implementation details}
\noindent \textbf{Datasets.} We experiment on the following five binary classification datasets: Breast cancer \cite{breast}, Blood transfusion \cite{blood}, Heart disease\cite{heart}, Raisin grains \cite{ccinar2020classification}, and Caesarian Section \cite{caesarian}. Additionally, we use six multi-class classification datasets: Glass identification \cite{glass}, Wine \cite{wine}, Car evaluation \cite{car}, Iris \cite{iris}, Seeds \cite{seeds} and Ecoli \cite{ecoli} to further confirm the capability of our method in handling multi-class classification tasks. All the above mentioned datasets for both binary and multi-class classifications are directly taken from the University of California \cite{blake1998uci}. We manually create our own training/test data split since they are not officially provided. The details of each dataset are summarized in Tab. \ref{table1}.

\begin{table}
\caption{Parameters of the optimization algorithms.}
\label{table2}
\resizebox{\linewidth}{!}{
\begin{tabular}{lll}
\toprule
\textbf{Optimizer} & \textbf{Parameter}                      & \textbf{Value}                                                                                                \\ \hline
BP     & learning rate                  & 0.01                                                                                                 \\
       & maximum iterations $R$            & 3000                                                                                                 \\
BBO    & modification probability       & 1                                                                                                    \\
       & immigration probability bounds & {[}0, 1{]}                                                                                           \\
       & step size                      & 1                                                                                                    \\
       & maximum immigration rate       & 1                                                                                                    \\
       & maximum migration rate         & 1                                                                                                    \\
       & mutation probability           & 0.1                                                                                                  \\
       & population size                & 100                                                                                                  \\
       & maximum iterations $R$            & \begin{tabular}[c]{@{}l@{}}300 for Breast, Blood, Wine, \\ Car and Iris; 400 for others\end{tabular} \\
GA     & encoding strategy              & binary gray                                                                                          \\
       & crossover probability          & 1                                                                                                    \\
       & mutation probability           & 0.01                                                                                                 \\
       & population size                & 100                                                                                                  \\
       & maximum iterations $R$            & \begin{tabular}[c]{@{}l@{}}300 for Breast, Blood, Wine, \\ Car and Iris; 400 for others\end{tabular} \\
PSO    & inertia weight                 & 1                                                                                                    \\
       & cognitive coefficient          & 0.3                                                                                                  \\
       & social coefficient             & 0.3                                                                                                  \\
       & population size                & 200                                                                                                  \\
       & maximum iterations $R$            & \begin{tabular}[c]{@{}l@{}}300 for Breast, Blood, Wine, \\ Car and Iris; 400 for others\end{tabular} \\
PBIL   & learning rate                  & 0.05                                                                                                 \\
       & negative learning rate         & 0.05                                                                                                 \\
       & num. of best individuals       & 1                                                                                                    \\
       & num. of bad populations        & 0                                                                                                    \\
       & population size                & 200                                                                                                  \\
       & maximum iterations $R$            & \begin{tabular}[c]{@{}l@{}}300 for Breast, Blood, Wine, \\ Car and Iris; 400 for others\end{tabular} \\
ES     & global variance                & 1                                                                                                    \\
       & num. of new individuals        & 10                                                                                                   \\
       & population size                & 250                                                                                                  \\
       & maximum iterations $R$            & \begin{tabular}[c]{@{}l@{}}300 for Breast, Blood, Wine, \\ Car and Iris; 400 for others\end{tabular} \\ \bottomrule
\end{tabular}}
\end{table}

\noindent \textbf{Implementation details.} MODN involves three key hyperparameters: dendrite number $M$, synaptic scaling factor $\alpha^{s}$, and telodendron scaling factor $\alpha^{t}$. Since we notice that the classification performance is sensitive to these three hyperparameters, similar to \cite{gao2019dendritic, luo2021decision, gao2021fully}, we carefully tune them on each dataset to hopefully achieve the best classification accuracy. See Sec. \ref{statis} for a detailed analysis. The parameter settings for each optimization algorithm are listed in Tab. \ref{table2}. For all the heuristic algorithms, we directly adopt the originally provided configurations from the corresponding literature or the open source library \footnote{Available from http://geatpy.com/index.php/home/}.  

\subsection{Quantitative evaluation} \label{evaluation}
We first investigate the performance of MODN under six types of optimization strategies. To gain a better understanding, we also assess two instances of MODN, namely MODNP and MODNF. For fair comparisons, the involved parameters ($M, \alpha^{s}, \alpha^{t}$) = ($10\times C$, 10,1) are equally set for all three models. Tab. \ref{table3} shows the classification results. It can be observed that among all the optimization methods, BBO achieves the superior performance on most datasets (9 in 11), whereas BP is generally less capable of exploring a decent optimization. The primary reason is that, while resetting the zero-valued gradients to 1 eases the dilemma of ``dendrite dying'' for BP, the multiplication operation in Eqs. \ref{eq:eq17} to \ref{eq:eq20} can still result in excessively small gradients. 

Regarding the model-wise performance, it can be seen from  Tab. \ref{table3} that MODN generally outperforms both of the instances on most datasets (7 in 11). This suggests that enforcing a learnable $\boldsymbol{P}$ provides a more powerful learning configuration in jointly exploring the dendrite state in \textit{Exclusive state} and \textit{Communal state}, whereas fixing $\boldsymbol{P}$ does not achieve this. 
Still, there exist some datasets in which MODN does not perform well. We assume the reason that despite our 2-step optimization policy, conventional optimizers are not fully capable of handling the complex MODN structure with $\boldsymbol{P}$ involved. Based on the above observation, we select the BBO for optimization, and the MODN framework for classification learning in the following sections.

\begin{table*}
\centering
\caption{Results of ACC scores of six optimization methods for three models on 11 datasets.}
\label{table3}
\scalebox{1.2}{
\begin{tabular}{{lccccccc}}
\toprule
 & & \textbf{BBO} & \textbf{BP} & \textbf{GA} & \textbf{PSO} & \textbf{PBIL}  & \textbf{ES} \\ \hline
\multirow{3}{*}{Breast cancer} & MODN  & \textbf{0.9835}       & -      & 0.7452 & 0.9319 & 0.9112 & 0.7277         \\
                                   & MODNP & 0.9781 & 0.5178 & 0.7134 & 0.8233 & 0.8567 & 0.7961        \\
                                   & MODNF & 0.9766          & 0.5275 & 0.7706 & 0.9613 & 0.9333 & 0.7815 \\ \hdashline
\multirow{3}{*}{Blood transfusion} & MODN   & \textbf{0.7964}          & -      & 0.7411 & 0.7400 & 0.7409 & 0.7364         \\
                               & MODNP  & 0.7740          & 0.4680 & 0.7300 & 0.7402 & 0.7460 & 0.7358         \\
                               & MODNF  & 0.7953 & 0.5844 & 0.7358 & 0.7436 & 0.7387 & 0.7393  \\ \hdashline
\multirow{3}{*}{Heart disease} & MODN  & \textbf{0.7462}          & -      & 0.6703 & 0.6802 & 0.6824 & 0.6703           \\
                              & MODNP & 0.7066          & 0.4659 & 0.6718 & 0.6703 & 0.6788 & 0.6703           \\
                              & MODNF & 0.7201 & 0.4545 & 0.6729 & 0.6824 & 0.6941 & 0.6703  \\ \hdashline
\multirow{3}{*}{Raisin grains} & MODN   & \textbf{0.8383}          & -      & 0.6461 & 0.8122 & 0.8033 & 0.6216           \\
                               & MODNP  & 0.8365          & 0.5122 & 0.6444 & 0.8022 & 0.7722 & 0.6770          \\
                               & MODNF  & 0.8376 & 0.5431 & 0.6661 & 0.8300 & 0.8009 & 0.6756 \\ \hdashline
\multirow{3}{*}{Caesarian section} & MODN   & 0.6000 & -      & 0.5750 & 0.6479 & 0.6479          & 0.3896         \\
                                   & MODNP  & 0.5604 & 0.4104 & 0.6250 & 0.6271 & \textbf{0.6708} & 0.3917        \\
                                   & MODNF  & 0.5688 & 0.4646 & 0.6146 & 0.6062 & 0.6646          & 0.4188 \\ \hdashline
\multirow{3}{*}{Glass identification} & MODN   & \textbf{0.4682}          & -      & 0.2969 & 0.3302 & 0.3287 & 0.3287       \\
                                      & MODNP  & 0.3829          & 0.2821 & 0.2922 & 0.2992 & 0.2922 & 0.3202      \\
                                      & MODNF  & 0.4651 & 0.1876 & 0.3008 & 0.3566 & 0.3310 & 0.3194  \\ \hdashline
\multirow{3}{*}{Wine} & MODN   & \textbf{0.7426}          & -      & 0.2778 & 0.4685 & 0.4278 & 0.4509  \\
                      & MODNP  & 0.6769 & 0.3519 & 0.2694 & 0.4305 & 0.3676 & 0.4398          \\
                      & MODNF  & 0.6537          & 0.3028 & 0.2759 & 0.5185 & 0.4231 & 0.4463           \\ \hdashline
\multirow{3}{*}{Car evaluation} & MODN   & 0.7567          & -      & 0.7065 & 0.7222 & 0.7160 & 0.7067         \\
                                & MODNP  & \textbf{0.7975} & 0.2453 & 0.7114 & 0.7216 & 0.7200 & 0.7074         \\
                                & MODNF  & 0.7440          & 0.3801 & 0.7129 & 0.7184 & 0.7177 & 0.7084  \\ \hdashline
\multirow{3}{*}{Iris} & MODN   & 0.8506          & -      & 0.6011 & 0.8361 & 0.7422 & 0.5300         \\
                      & MODNP  & \textbf{0.9433} & 0.3072 & 0.6111 & 0.8667 & 0.7728 & 0.5128 \\
                      & MODNF  & 0.7556          & 0.2800 & 0.5806 & 0.6917 & 0.6711 & 0.6528        \\ \hdashline
\multirow{3}{*}{Seeds} & MODN   & 0.8976          & -      & 0.4302 & 0.6833 & 0.5929 & 0.5437        \\
                       & MODNP  & \textbf{0.9286} & 0.3810 & 0.5095 & 0.6865 & 0.6952 & 0.5889  \\
                       & MODNF  & 0.7540          & 0.4286 & 0.5246 & 0.6341 & 0.6246 & 0.5540          \\ \hdashline
\multirow{3}{*}{Ecoli} & MODN   & 0.6259 & -      & 0.3566 & 0.5003 & 0.4696 & \textbf{0.6865}       \\
                       & MODNP  & 0.5710 & 0.2737 & 0.3623 & 0.4121 & 0.4552 & 0.6581               \\
                       & MODNF  & 0.4700 & 0.2296 & 0.3848 & 0.5081 & 0.4875 & 0.3694           \\ \bottomrule
\end{tabular}}
\end{table*}

\subsection{Comparison against other methods} 
We here quantitatively evaluate the classification performance of MODN against prior methods. Specifically, we compare against two baseline artificial neural network models: MLP and DNM \cite{gao2019dendritic} (with BBO as optimizer). Since we only focus on investigating the performance under shallowly-structured models, we create a 2-layer MLP for fair comparisons. The hidden layer dimensionality for MLP is set with the dendrite number $M$ to be consistent with MODN. Also, the MLP is optimized via the same loss function (i.e., cross-entropy) with its final output layer designed as the softmax activation, as ours. As for DNM, we set $(\alpha^{t}, \alpha^{s})$ to the same value as ours for all datasets considering the structure resemblance. We use the fine-tuned value of $\theta$ for DNM as suggested in \cite{gao2019dendritic}. All hyperparameters of MODN (after fine-tuning) are listed in Tab. \ref{table4}. 

\noindent \textbf{Evaluation metrics.} The overall performance evaluation is measured with the average accuracy (\textbf{ACC}) and area under the ROC curve (\textbf{AUC}) metrics over 30 runs. ACC is calculated by $\rm{ACC} = \sum_{n=1}^{N}\mathbbm{1}(\widehat{O}^n = \tau(O^n))$, in which $\tau(\cdot)$ is a mapping that one-hot vectorizes our network output $O^n$. Then, AUC is calculated by the confusion matrix composed of true positive (TP), true negative (TN), false positive (FP), and false negative (FN). TP and TN represent the numbers of correctly classified positive and negative samples, respectively. FP refers to the number of positive samples incorrectly classified, whereas FN refers to the number of negative samples correctly labeled. In terms of binary classification, for any selected positive class, we calculate the true positive rate (\textbf{TPR}, also termed as recall) via $\rm{TPR = TP/(TP+FN)}$, and the false positive rate (\textbf{FPR}) via $\rm{FPR = FP/(FP+TN)}$. While for multi-class classification, we follow the micro-averaging policy by deriving \textbf{TPR} and \textbf{FPR} with $\rm{TPR} = \sum_c\rm{TP}_c/ \sum_c(\rm{TP}_c+\rm{FN}_c)$ and $\rm{FPR} = \sum_c\rm{FP}_c/ \sum_c(\rm{FP}_c+\rm{TN}_c)$, where $c =1, \cdots, C$ is the assumed one positive class, leaving the remaining classes to be negative. The obtained TRPs and FPRs for both binary and multi-class cases are then set to y- and x-axes, respectively to plot the ROC curve. 

To further systematically evaluate the performance of all three methods, we adopt the Wilcoxon signed-rank test (WSRT) \cite{garcia2009study} as an assessment criteria of the significant difference between ours and the compared models. Assume, in the $i$-th run, the paired experimental results (ACC) of two methods (i.e., one for ours and the other for the compared model) for calculating WSRT are $A_i^1$ and $A_i^2$. Then, WSRT calculates a positive-negative signed-rank sum pair ($\mathbf{W^{+}}, \mathbf{W^{-}}$) as the test statistic via

\begin{equation}
\label{eq:eq25}
\left\{
\begin{aligned}
\mathbf{W^{+}}=\sum_{i=1}^{30}\rm{sgn}(A_i^2-A_i^1)R_{i},  \quad   A_i^2-A_i^1>0, \\
\mathbf{W^{-}}=\sum_{i=1}^{30}\rm{sgn}(A_i^2-A_i^1)R_{i},  \quad   A_i^2-A_i^1<0,
\end{aligned}
\right.
\end{equation}
where $R_{i}$ indicates the rank of the absolute value of the difference between $A_i^2$ and $A_i^1$, and $\rm{sgn}()$ denotes the sign function. When the difference between the two paired samples are equal, the ranks computed in both runs are re-assigned with the averaged ranks. In addition, if $A_i^1$ is equal to $A_i^2$, then the resulting pair (both $A_i^1$ and $A_i^2$) will be removed from the rank calculation. 
After obtaining ($\mathbf{W^{+}}, \mathbf{W^{-}}$), WSRT computes the final statistic $-$ p-value as the significance indicator. Considering the large number of result pairs, we use normal distribution approximation for p-value calculation. In particular, we follow the standard WSRT protocols by first estimating the Gaussian parameters via $\hat{\mu} = T(T+1)/4$ and $\hat{\delta} = \sqrt{T(T+1)(2T+1)/24}$, where $T$ refers to the number of result pairs with non-zero differences. Then, the p-value $p$ can be derived by computing $p = P(X \leq \rm{min}\{\mathbf{W^{+}}, \mathbf{W^{-}}\})$, where $ X \sim \mathcal{N}(\hat{\mu},\hat{\delta}^2)$. More precisely, we report the adjusted p-value with the Bonferroni correction to modify $p_{adjusted}$ from $p_{unadjusted}$ to $2p_{unadjusted}$, because the significance test involves two pairs of statistical tests (i.e., comparing MODN with MLP and comparing MODN with DNM). Theoretically, a p-value smaller than 0.05 represents the difference between two compared methods is significant at the 0.05 level. Besides, a larger $\mathbf{W^{+}}$ and a smaller $\mathbf{W^{-}}$ can also serve as the reference of significance.



The quantitative results are reported in Tab. \ref{table5}. It can be confirmed that MODN outperforms MLP by a large margin in all datasets (i.e., both tasks) with respect to AUC and ACC (3rd and 4th columns in Tab. \ref{table5}). The reason can be attributed to the fact that dendrite neural families intrinsically involve a mathematically modeled synapse layer that is capable of addressing the complex XOR problem. As for DNM, we can observe it does not handle the class-imbalanced datasets well (i.e., Breast cancer, Blood transfusion). Compared with DNM, the preference filter in MODN can retain the specific feature with selected dendrites from the minor data mode in classifying on imbalanced datasets. Furthermore, we can observe that the p-values in Tab. \ref{table5}(last column) are generally smaller than $10^{-6}$, which indicates that the obtained superior statistics in Tab. \ref{table5} are sufficiently reliable.



\begin{table}
\centering
\caption{Tuned hyperparameters ($M, \alpha^s, \alpha^t$) of MODN in all classification datasets.}
\label{table4}
\resizebox{!}{!}{
\begin{tabular}{lccc}
\toprule
\textbf{Dataset} & $M$ & $\alpha^s$ & $\alpha^t$ \\ \hline
Breast cancer        & 24       & 8                   & 1.5 \\
Blood transfusion    & 20       & 10                  & 1   \\
Heart disease        & 48       & 8                   & 1.5 \\
Raisin grains        & 8        & 20                  & 0.1   \\
Caesarian Section    & 16       & 1                   & 0.9 \\
Glass identification & 54       & 10                  & 1   \\
Wine                 & 30       & 10                  & 1   \\
Car evaluation       & 40       & 10                  & 1   \\
Iris                 & 12       & 10                  & 1   \\
Seeds                & 16       & 5                   & 1   \\
Ecoli                & 30       & 10                  & 1.5   \\ \bottomrule
\end{tabular}}
\end{table}

\begin{table*}
\centering
\caption{Quantitative evaluation of MODN against MLP and DNM on 11 datasets.}
\label{table5}
\scalebox{1.2}{
\begin{tabular}{llccrrc}
\toprule
\textbf{Datasets}                     & \textbf{Methods}  & \textbf{AUC}$\uparrow$    & \textbf{ACC}$\uparrow$    & $\mathbf{W^{+}}\uparrow$ & $\mathbf{W^{-}}\downarrow$ & \textbf{p-value (0.05)$\downarrow$} \\ \hline
\multicolumn{7}{c}{\textbf{Binary classification}} \\ \hline
\multirow{3}{*}{Breast cancer}        & MLP          & 0.8364          & 0.8397          & 465.0       & 0.0         & \underline{3.602968e-06}     \\
                                      & DNM          & 0.9966          & 0.9824          & 110.5       & 42.5        & 2.153436e-01     \\
                                      & MODN (ours) & \textbf{0.9972} & \textbf{0.9847} & -           & -           & -                \\ \hdashline
\multirow{3}{*}{Blood transfusion}    & MLP          & 0.7660          & 0.7393          & 465.0       & 0.0         & \underline{3.508982e-06}     \\
                                      & DNM          & 0.7532          & 0.7860          & 355.0       & 110.0        & \underline{2.379320e-02}     \\
                                      & MODN (ours) & \textbf{0.8475} & \textbf{0.7964} & -           & -           & -                \\ \hdashline
\multirow{3}{*}{Heart disease}         & MLP          & 0.6011          & 0.6612          & 464.0       & 1.0         & \underline{7.450580e-09}     \\
                                      & DNM          & \textbf{0.9063} & 0.7853          & 327.0       & 108.0       & \underline{3.656010e-02}     \\
                                      & MODN (ours) & 0.9031          & \textbf{0.8147} & -           & -           & -                \\ \hdashline
\multirow{3}{*}{Raisin grains}        & MLP          & 0.6645          & 0.5854          & 465       & 0         & \underline{3.634760e-06}     \\
                                      & DNM          & \textbf{0.9426} & \textbf{0.8793} & 4.0         & 461.0       & \underline{8.172220e-06}     \\
                                      & MODN (ours) & 0.9116          & 0.8483          & -           & -           & -                \\ \hdashline
\multirow{3}{*}{Caesarian section}    & MLP          & 0.4978          & 0.4521          & 465.0       & 0.0         & \underline{3.445290e-06}     \\
                                      & DNM          & \textbf{0.7189} & 0.5479          & 405.0       & 30.0         & \underline{9.352932e-05}     \\
                                      & MODN (ours) & 0.6552          & \textbf{0.6917} & -           & -           & -                \\ \hline
\multicolumn{7}{c}{\textbf{Multi-class classification}} \\ \hline
\multirow{2}{*}{Glass identification} & MLP          & 0.6312          & 0.2930          & 463.5       & 1.5         & \underline{2.118752e-06}     \\
                                      & MODN (ours) & \textbf{0.8407} & \textbf{0.4985} & -           & -           & -                \\ \hdashline
\multirow{2}{*}{Wine}                 & MLP          & 0.7151          & 0.5389          & 429.5       & 5.5        & \underline{4.745973e-06}     \\
                                      & MODN (ours) & \textbf{0.8960} & \textbf{0.7426} & -           & -           & -                \\ \hdashline
\multirow{2}{*}{Car evaluation}       & MLP          & 0.7842          & 0.6683          & 119.0        & 1.0         & \underline{8.876252e-04}     \\
                                      & MODN (ours) & \textbf{0.9412} & \textbf{0.7567} & -           & -           & -                \\ \hdashline
\multirow{2}{*}{Iris}                 & MLP          & 0.6349          & 0.4117          & 435.0       & 0.0         & \underline{2.609510e-06}     \\
                                      & MODN (ours) & \textbf{0.9908} & \textbf{0.9300} & -           & -           & -                \\ \hdashline
\multirow{2}{*}{Seeds}                & MLP          & 0.6882          & 0.4944          & 464.0       & 1.0         & \underline{2.008739e-06}     \\
                                      & MODN (ours) & \textbf{0.9811} & \textbf{0.9143} & -           & -           & -                \\ \hdashline
\multirow{2}{*}{Ecoli}                & MLP          & 0.5657          & 0.2872          & 435.0       & 0.0         & \underline{2.682501e-06}     \\
                                      & MODN (ours) & \textbf{0.8750} & \textbf{0.6313} & -           & -           & -                \\ \bottomrule
\end{tabular}}
\end{table*}

\subsection{Model analysis} \label{statis}
In this section, we provide further insights into MODN by analyzing the parameter sensitivity, convergence, and the filtering matrix. We first evaluate the sensitivity of the hyperparameters for MODN on one binary dataset Raisin and one multi-class classification dataset Iris as two examples. Then, we perform convergence analysis of MODN to inspect the model-optimizer performance combination. We eventually visualize the filtering matrix to verify the feasibility of our learnable design.


\noindent \textbf{Hyperparameter sensitivity.} MODN contains three hyperparameters: the dendrite number $M$, the synaptic scaling factor $\alpha^s$, and the telodendron scaling factor $\alpha^t$. To investigate the underlying sensitivity, we vary ($M$, $\alpha^s$, $\alpha^t$) to report the classification accuracy in Fig. \ref{fig:4} on the Raisin and Iris datasets as two examples. 

The dendrite number $M$ controls the structure of the model. We report in Tab. \ref{fig:4}(a,b) the changes of ACCs by traversing $M$ within [4,20] with fixed $\alpha^s$ and $\alpha^t$. We note that it requires to empirically set $M$ to achieve the best performance, and an overly large or small $M$ can lead to low learning efficiency or overfitting, respectively. Additionally, a smaller $\alpha^t$ and a larger $\alpha^s$ (Fig. \ref{fig:4}(c,d)) promote the ACC better. From the above analysis, we can confirm that $M$ is somewhat sensitive, whereas $\alpha^s$ and $\alpha^t$ exhibit a uniform sensitivity tendency for both datasets. Therefore, we follow \cite{gao2019dendritic} by tuning the hyperparameters (in Tab. \ref{table4}) in our experiments to ensure decent performance. 

\noindent \textbf{Convergence.} We next investigate the convergence of our model. Specifically, we plot the convergence curves of five model-optimizer combinations in Fig. \ref{fig:5} for evaluation. We can see in Fig. \ref{fig:5} that the curves of Eq. \ref{eq:eq10} with the BBO optimizer decline drastically at first, and then tend to be steady (i.e., converged) with increasing iterations. In particular, the MODN+BBO contributes to the earliest convergence (Fig. \ref{fig:5}(a,b,g,j,k)) and generally yields the lowest loss value (Fig. \ref{fig:5}(a,b,d$\sim$e)). We assume the reason is that our two-step learning scheme effectively eases the training. Also, BP-based optimization leads to a slower (blue line in Fig. \ref{fig:5}(f,i$\sim$k)) or failed convergence (purple line in Fig. \ref{fig:5}(d,g,i,j)).

\noindent \textbf{Visualization of $\boldsymbol{P}$.} To gain a deeper understanding of the filtering matrix $\boldsymbol{P}$, we visualize the entire matrix elements of $\boldsymbol{P}$ regarding MODN, MODNF, and MODNP to examine its role on Caesarian, Iris and Glass datasets. Fig. \ref{fig:6} shows the results. For MODN in Fig. \ref{fig:6}(a,d,g), we exhibit the learned $\boldsymbol{P}$ after our 2-step BBO optimization. We can confirm that MODN can capture the desired $\boldsymbol{P}$ to induce the \textit{exclusive state} and the \textit{communal state}. Furthermore, it is noteworthy that all the columns in the learned $\boldsymbol{P}$ Fig. \ref{fig:6}(a,d,g) successfully avoid the all-zero-vector case, indicating that every dendrite in MODN contributes to regressing the output. By contrast, because the filtering matrix $\boldsymbol{P}$ for MODNF and MODNP Fig. \ref{fig:6}(b,c,e,f,h,i) are pre-determined, they cannot adaptively select the beneficial dendrites during learning, which reflects on the low ACCs. We can thus confirm the effectiveness and feasibility of the learning design of $\boldsymbol{P}$ for downstream tasks.

\begin{figure}
\centering
\subfloat[$M$ on Raisin]{\label{fig4:a}\includegraphics[width=0.45\linewidth]{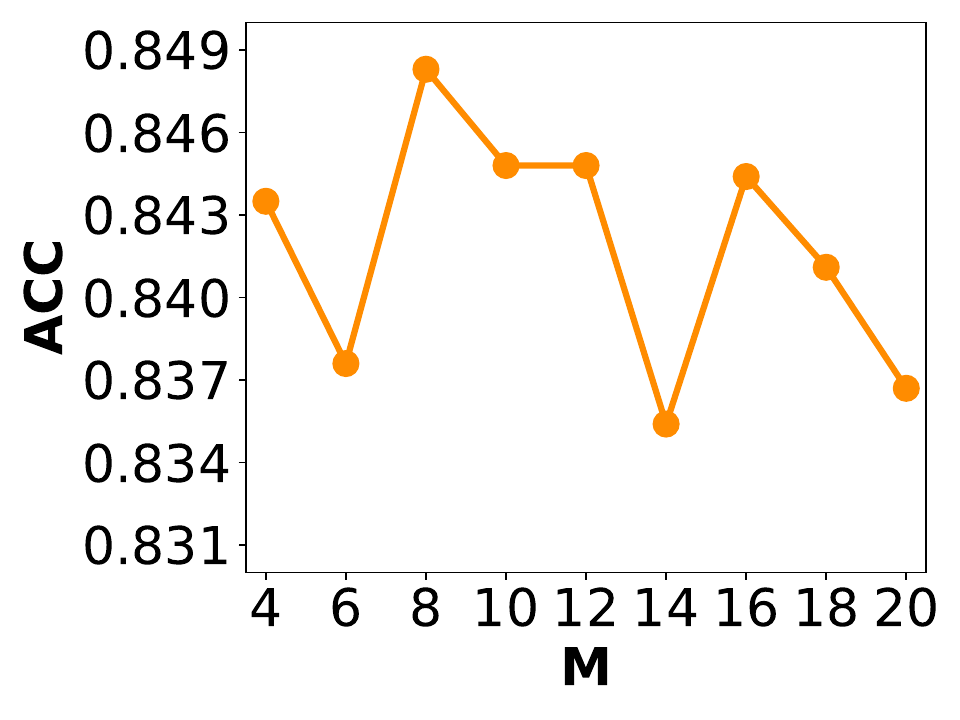}}\quad
\subfloat[$M$ on Iris]{\label{fig4:b}\includegraphics[width=0.45\linewidth]{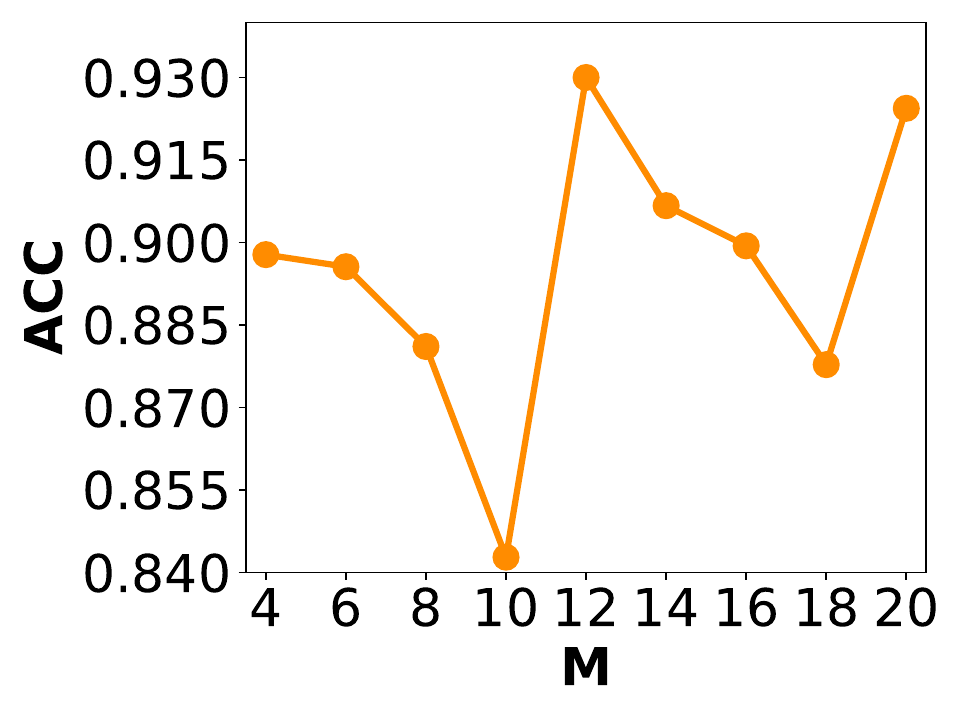}}\\
\subfloat[$\alpha^s$ and $\alpha^t$ on Raisin]{\label{fig4:c}\includegraphics[width=0.45\linewidth]{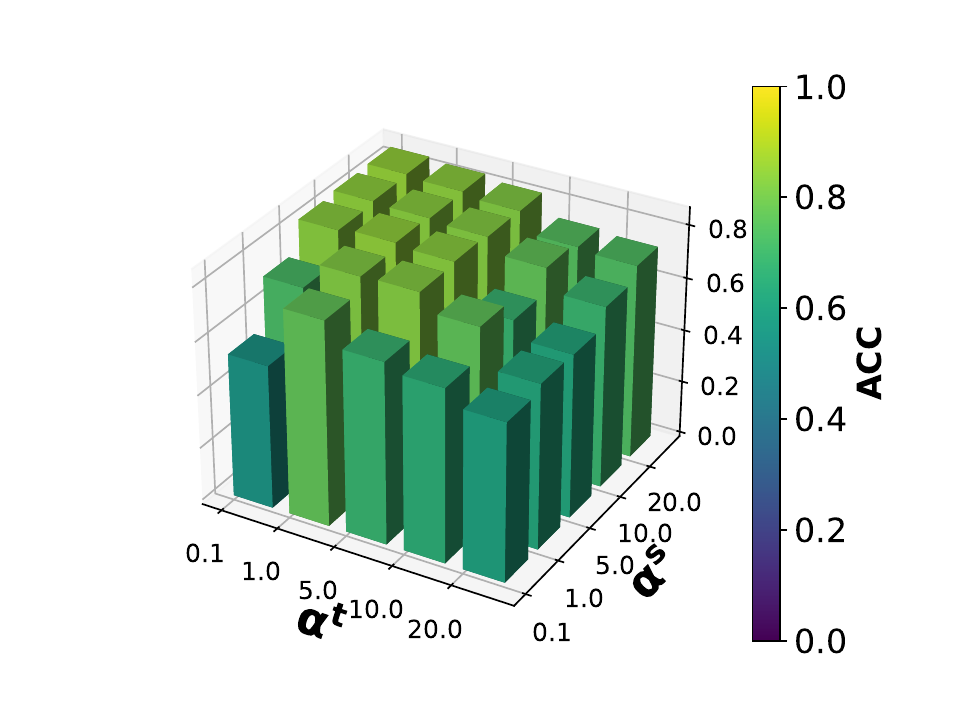}}\quad
\subfloat[$\alpha^s$ and $\alpha^t$ on Iris]{\label{fig4:d}\includegraphics[width=0.45\linewidth]{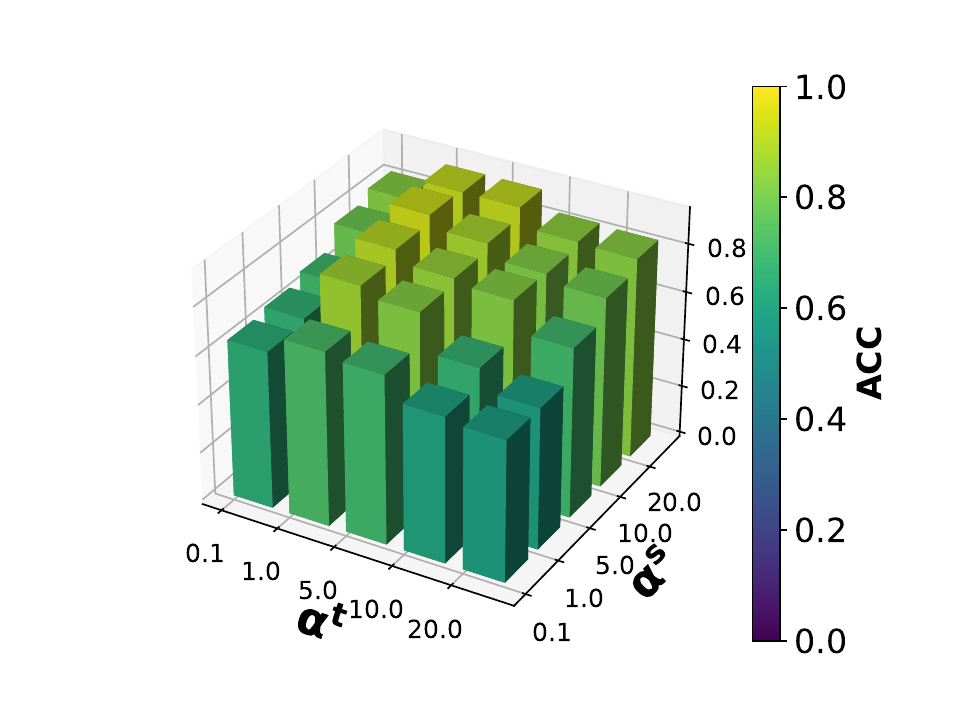}}\\
\caption{Performance evaluation of different parameter settings ($M, \alpha^s, \alpha^t$) on the Raisin and Iris datasets.}
\label{fig:4}
\end{figure}

\begin{figure*}
\centering
\subfloat[Breast]{\label{fig5:a}\includegraphics[width=0.22\linewidth]{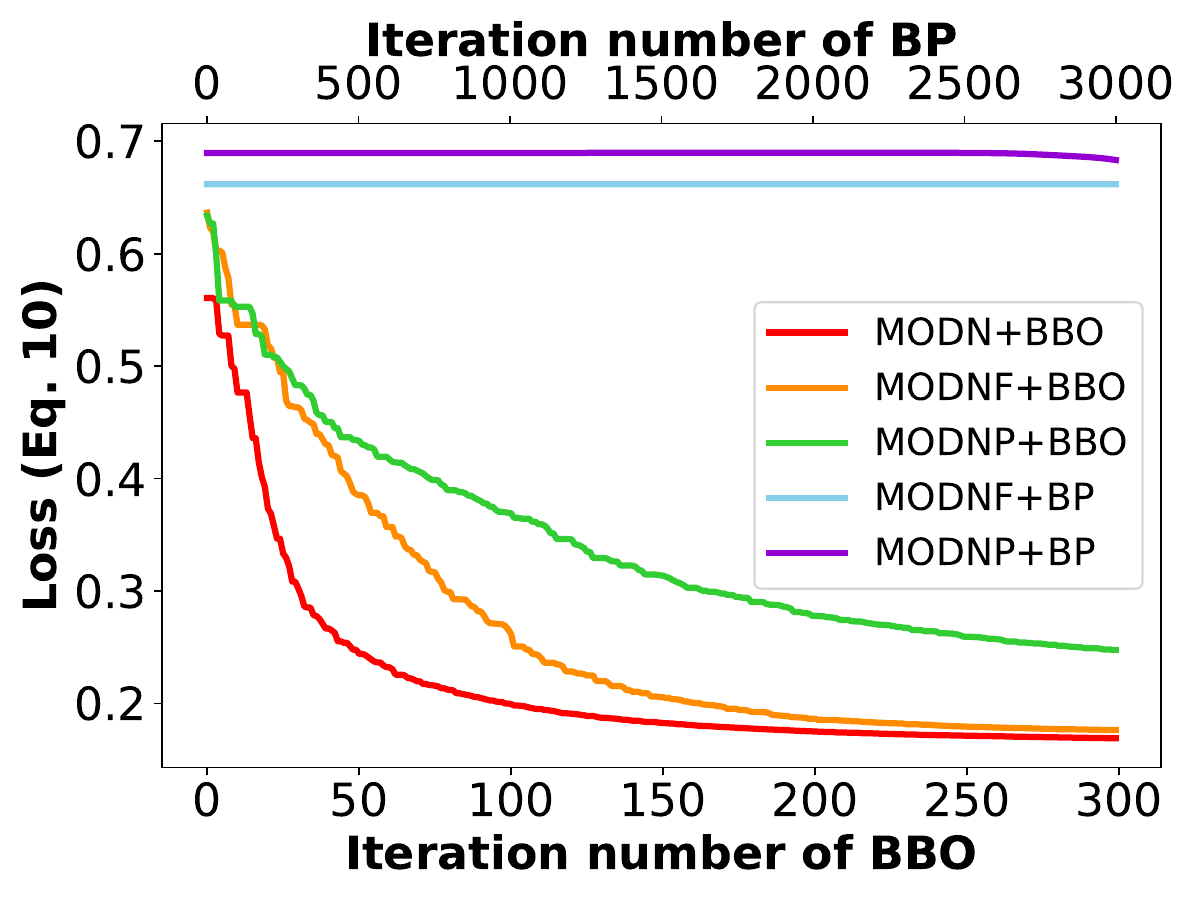}}\quad
\subfloat[Blood transfusion]{\label{fig5:b}\includegraphics[width=0.22\linewidth]{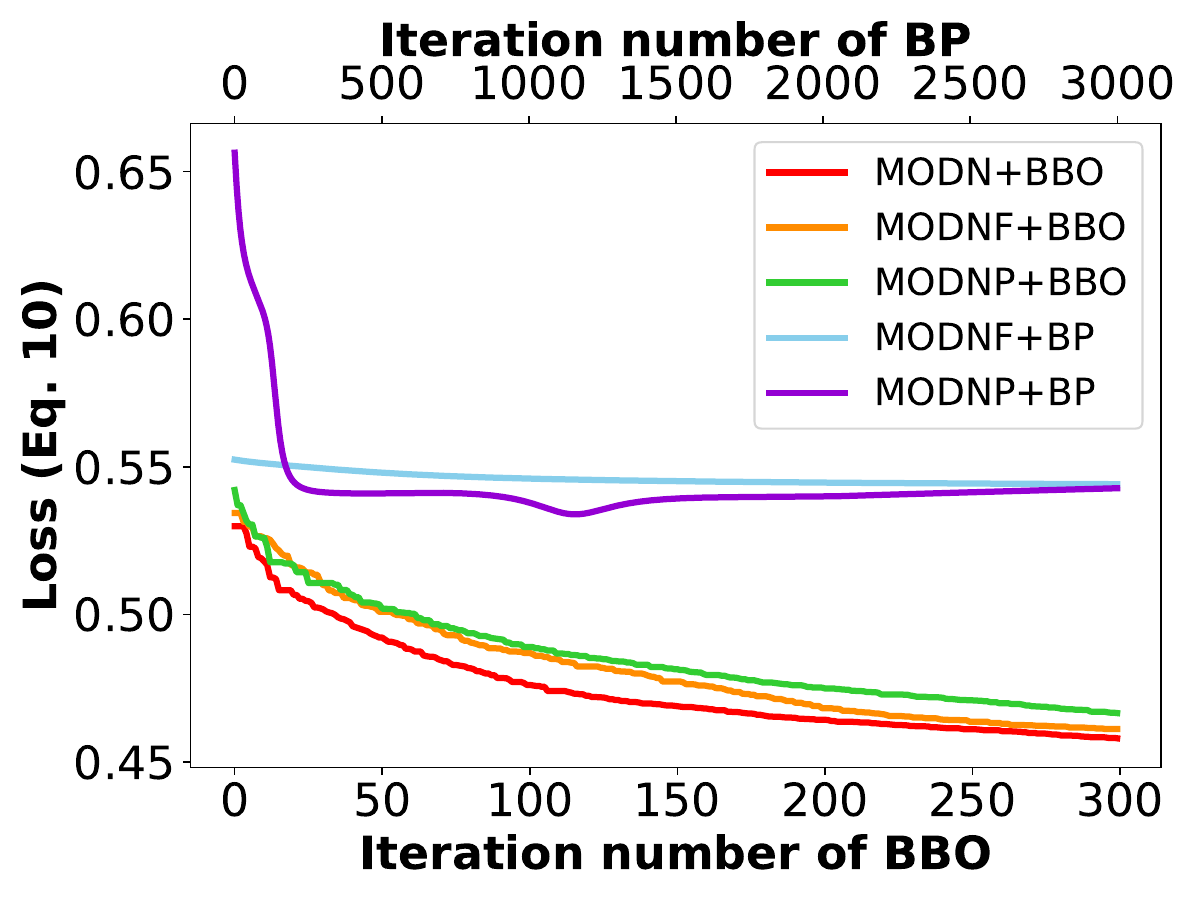}}\quad
\subfloat[Heart]{\label{fig5:c}\includegraphics[width=0.22\linewidth]{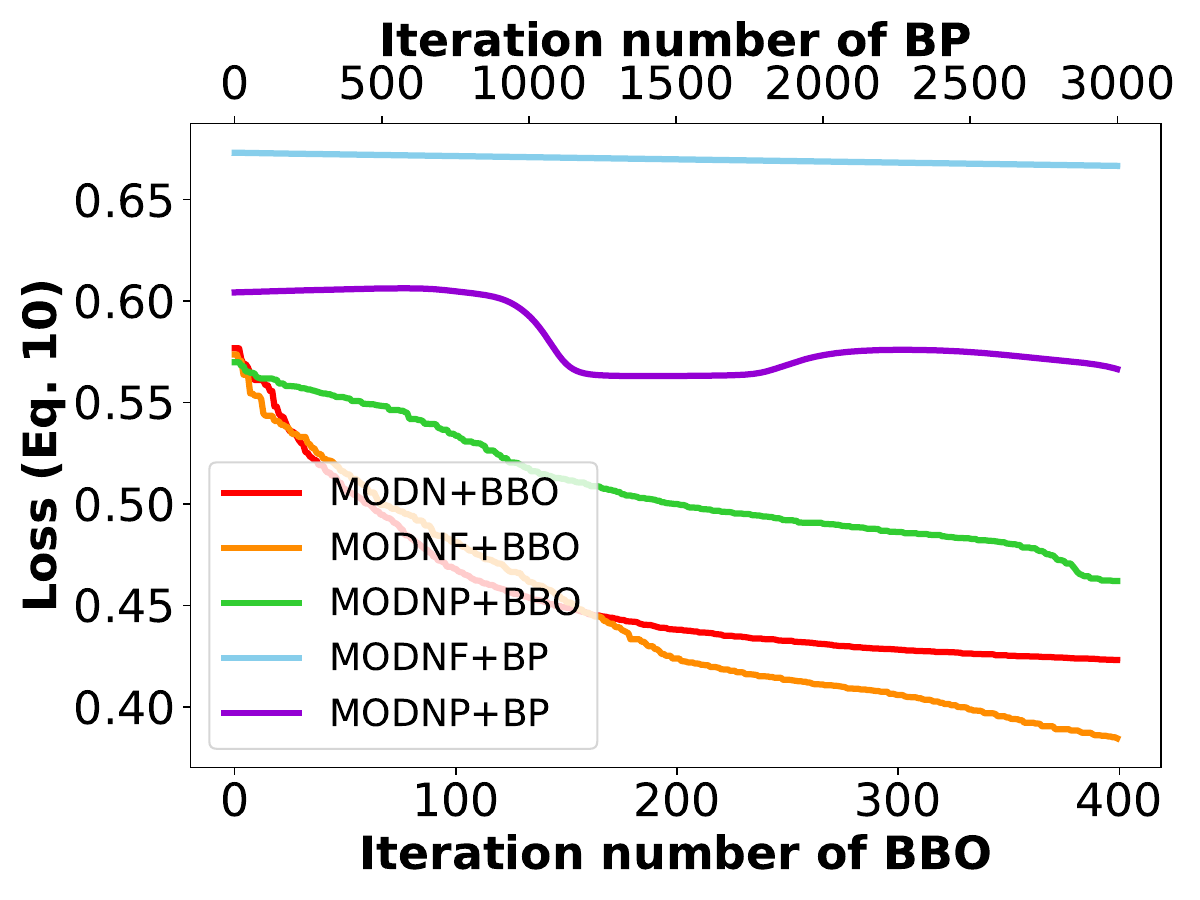}}\quad
\subfloat[Raisin]{\label{fig5:d}\includegraphics[width=0.22\linewidth]{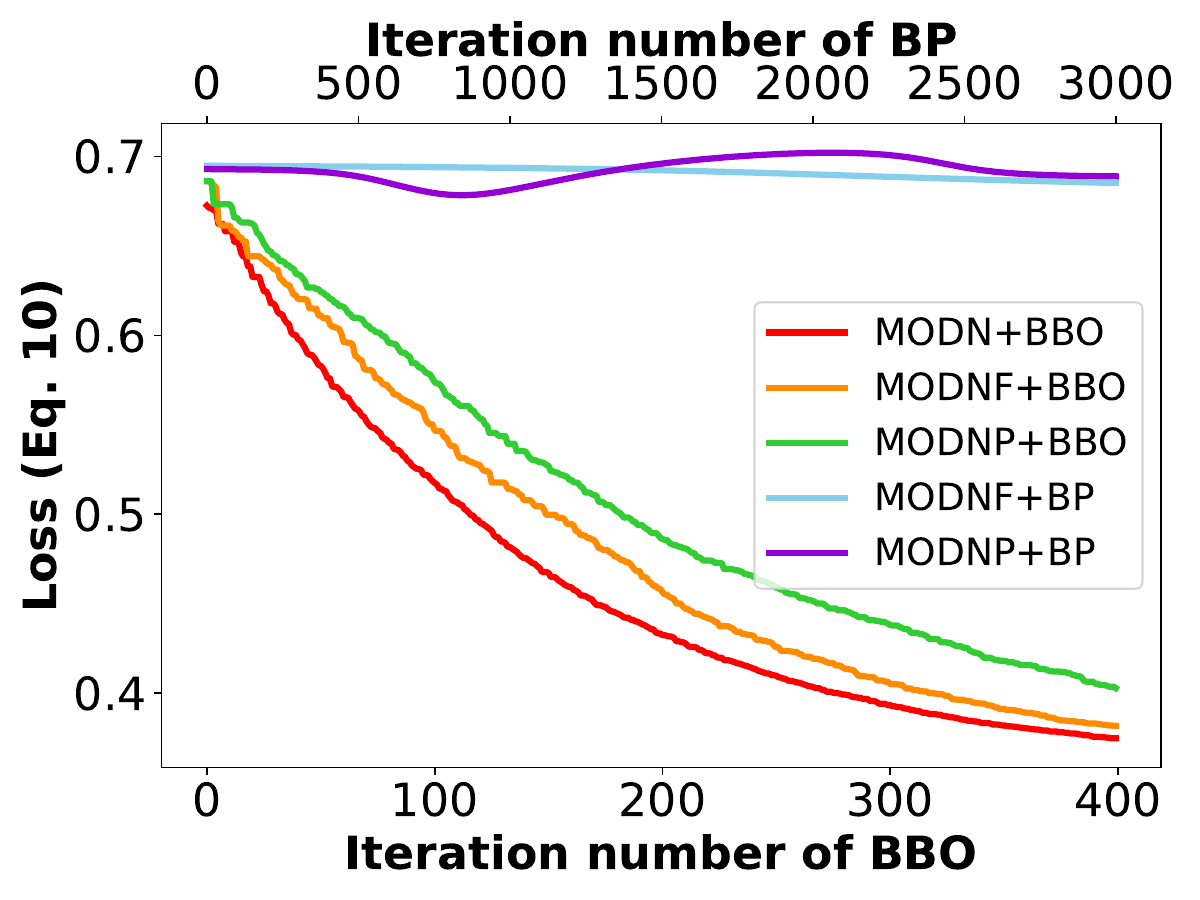}}\\
\subfloat[Caesarian]{\label{fig5:e}\includegraphics[width=0.22\linewidth]{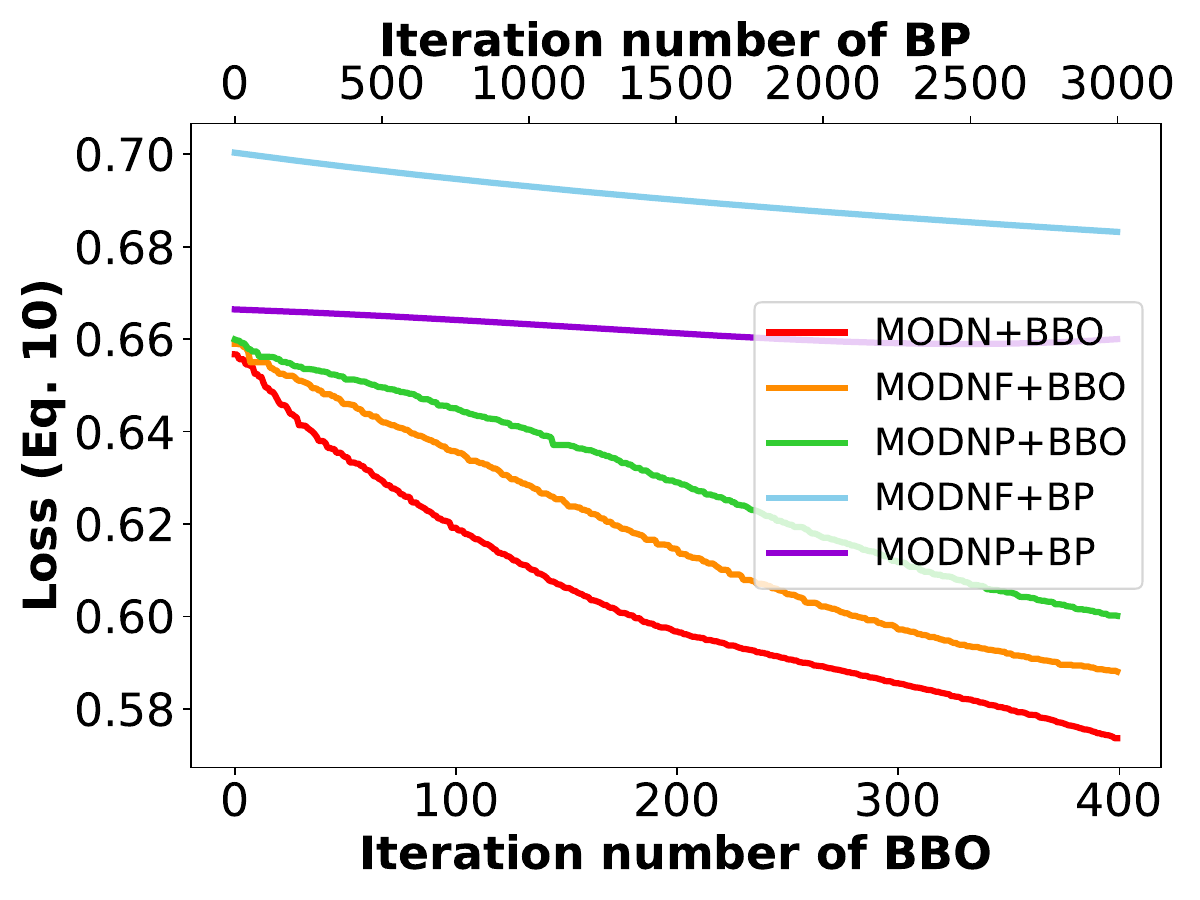}}\quad
\subfloat[Glass]{\label{fig5:f}\includegraphics[width=0.22\linewidth]{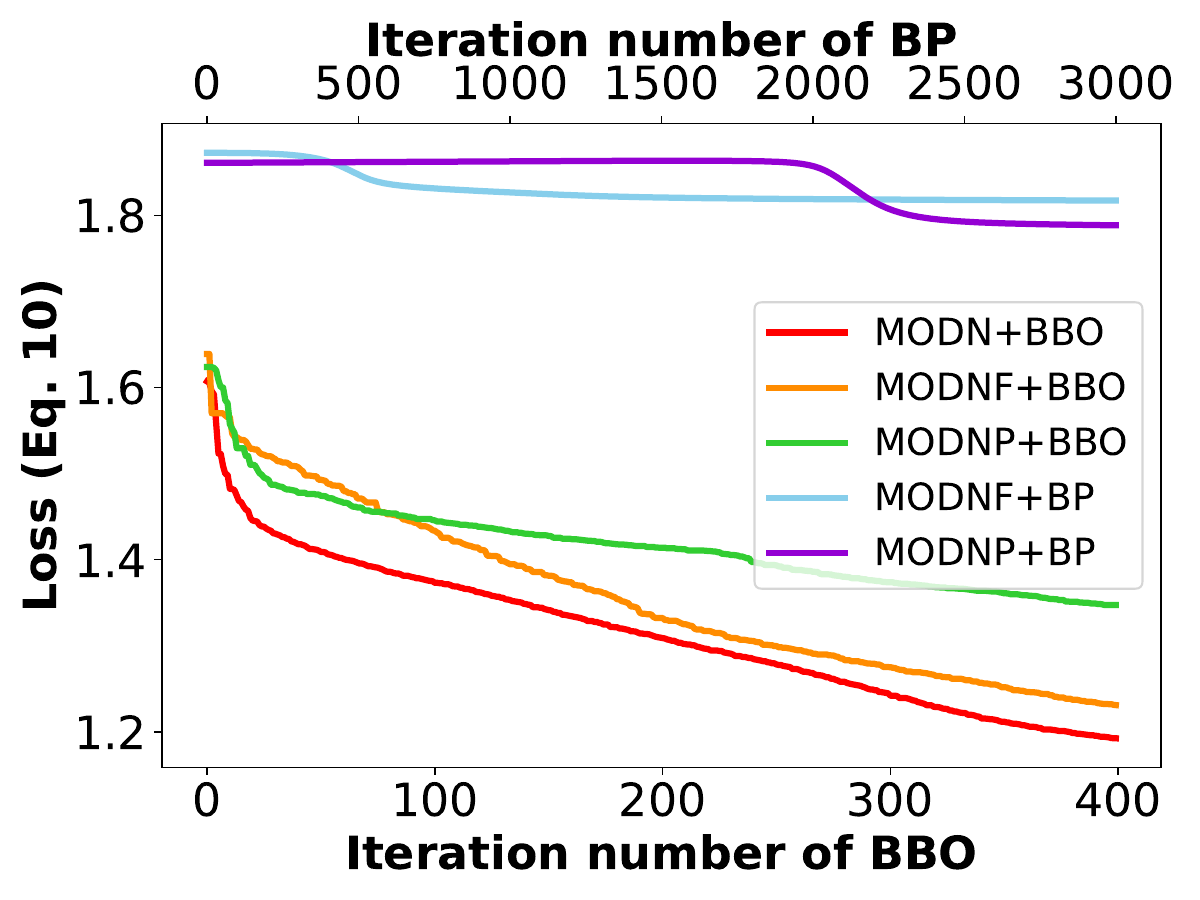}}\quad
\subfloat[Wine]{\label{fig5:g}\includegraphics[width=0.22\linewidth]{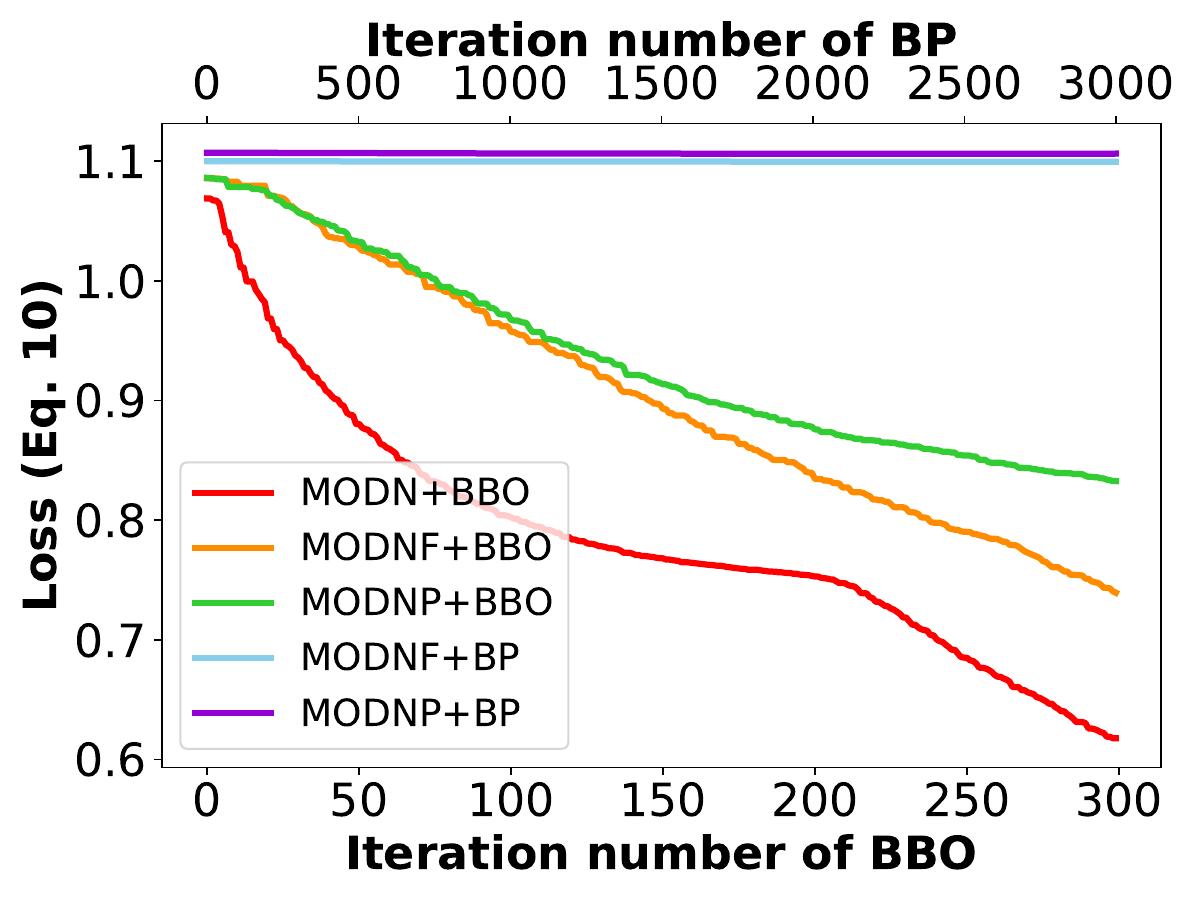}}\quad
\subfloat[Car]{\label{fig5:h}\includegraphics[width=0.22\linewidth]{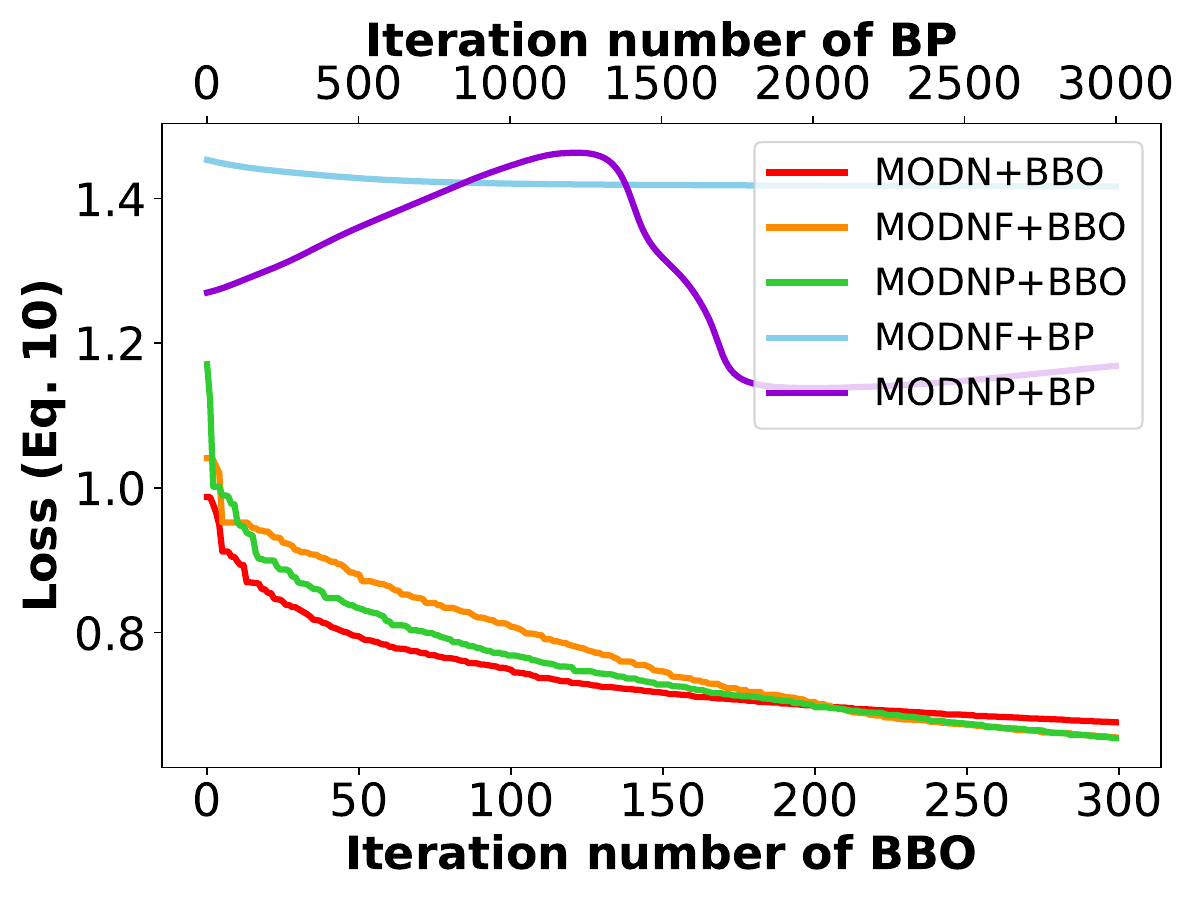}}\\
\subfloat[Iris]{\label{fig5:i}\includegraphics[width=0.22\linewidth]{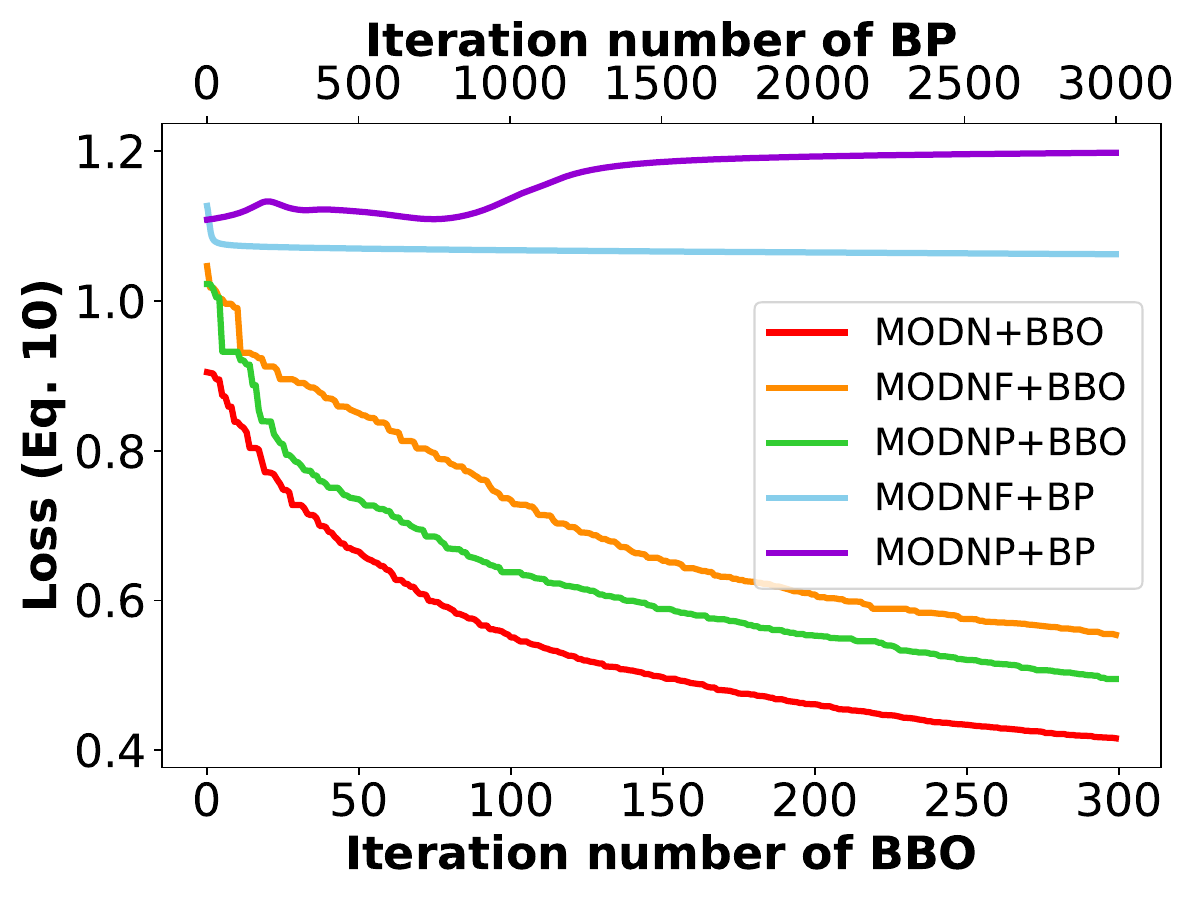}}\quad
\subfloat[Seeds]{\label{fig5:j}\includegraphics[width=0.22\linewidth]{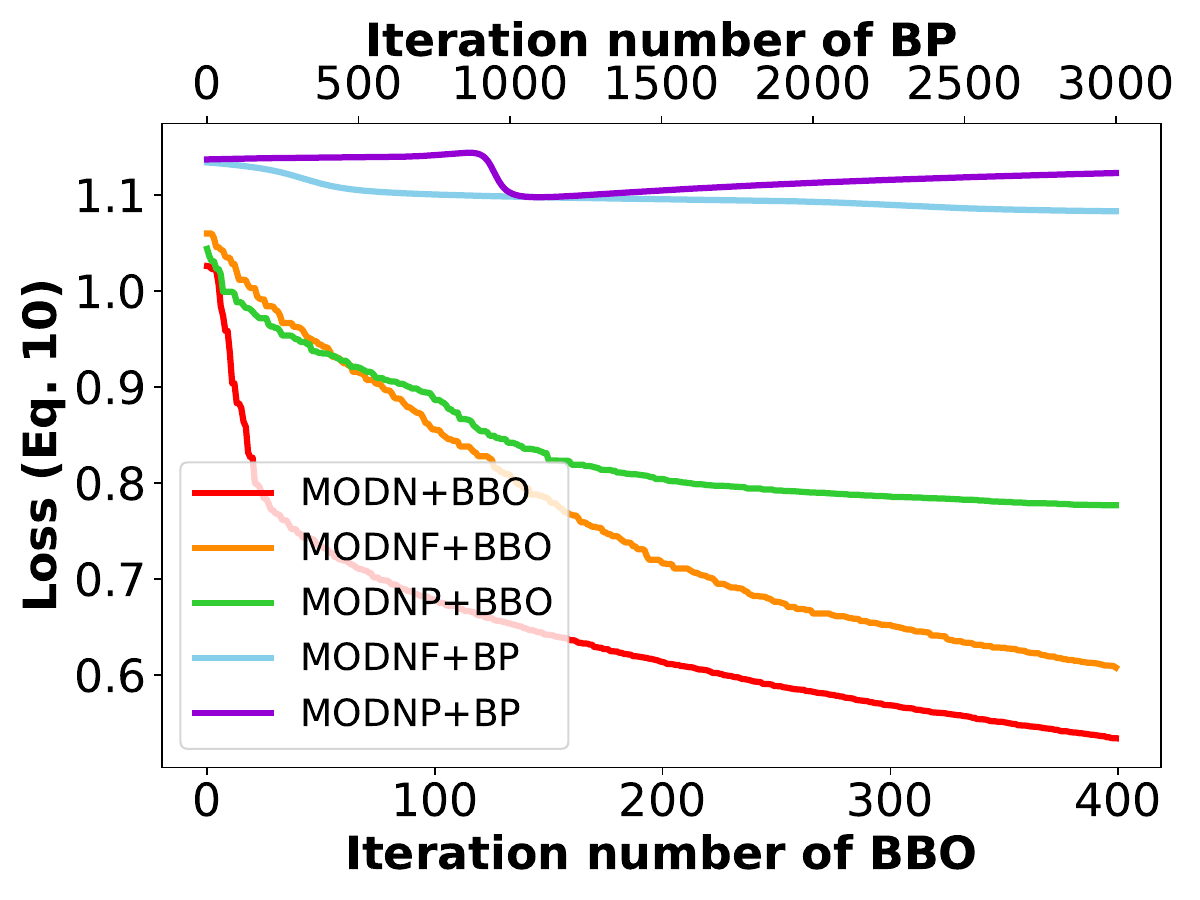}}\quad
\subfloat[Ecoli]{\label{fig5:k}\includegraphics[width=0.22\linewidth]{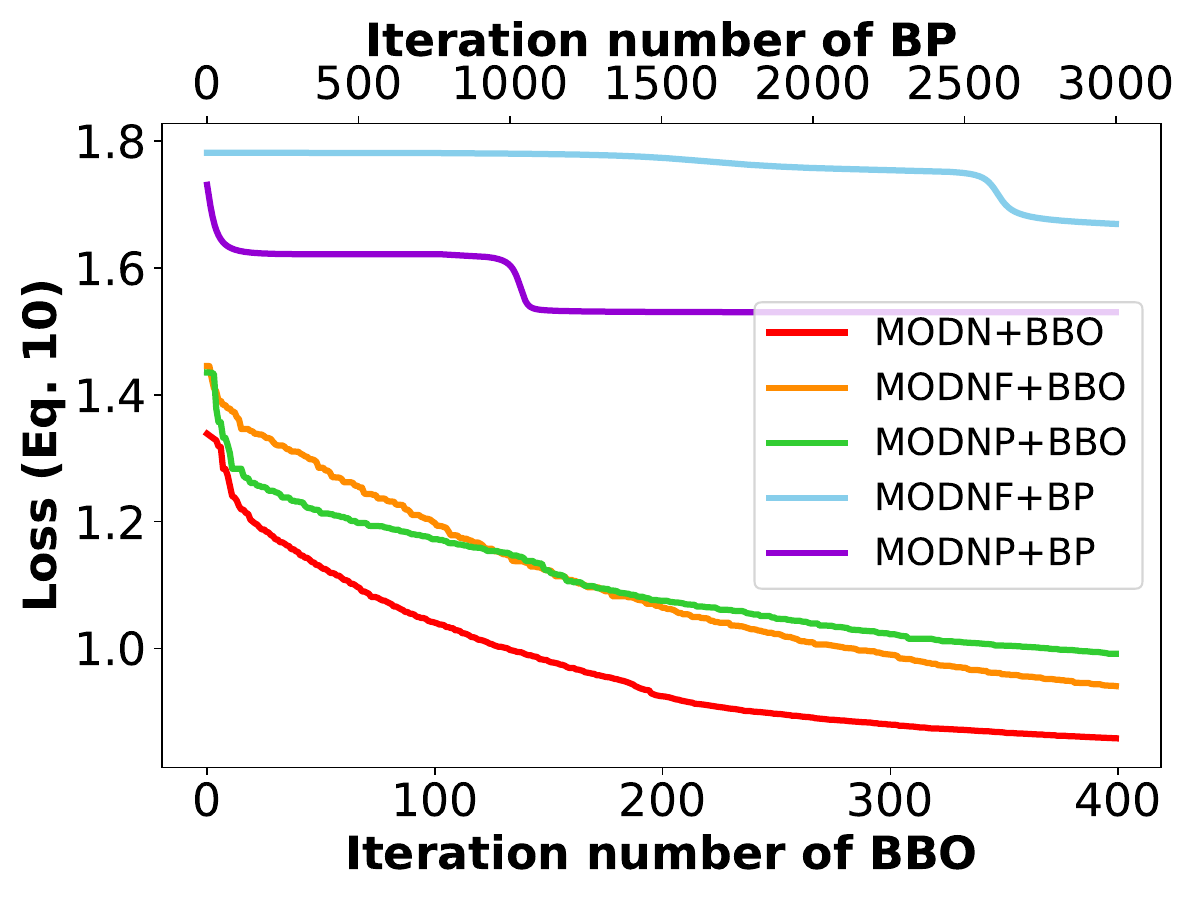}}\\
\caption{Curves of Eq. \ref{eq:eq10} with model-optimizer combinations on all datasets.}
\label{fig:5}
\end{figure*}

\begin{figure}
\centering
\subfloat[MODN \\ (ACC=0.8125)]{\label{fig6:a}\includegraphics[width=0.27\linewidth]{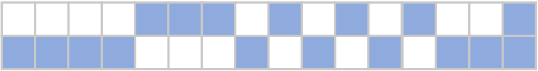}}\quad
\subfloat[MODNF \\ (ACC=0.7500)]{\label{fig6:b}\includegraphics[width=0.27\linewidth]{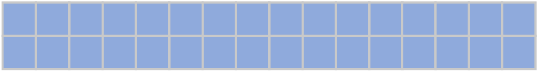}}\quad
\subfloat[MODNP \\ (ACC=0.4375)]{\label{fig6:c}\includegraphics[width=0.27\linewidth]{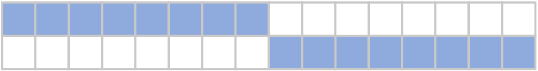}}\\

\subfloat[MODN \\ (ACC=0.9500)]{\label{fig6:d}\includegraphics[width=0.27\linewidth]{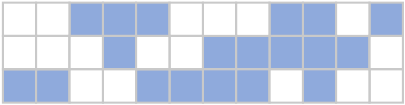}}\quad
\subfloat[MODNF \\ (ACC=0.9167)]{\label{fig6:e}\includegraphics[width=0.27\linewidth]{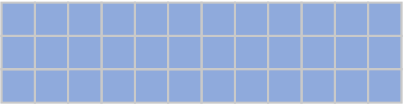}}\quad
\subfloat[MODNP \\ (ACC=0.9333)]{\label{fig6:f}\includegraphics[width=0.27\linewidth]{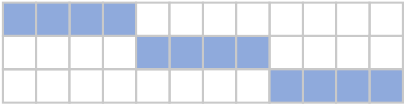}}\\

\subfloat[MODN (ACC=0.6744)]{\label{fig6:g}\includegraphics[width=0.9\linewidth]{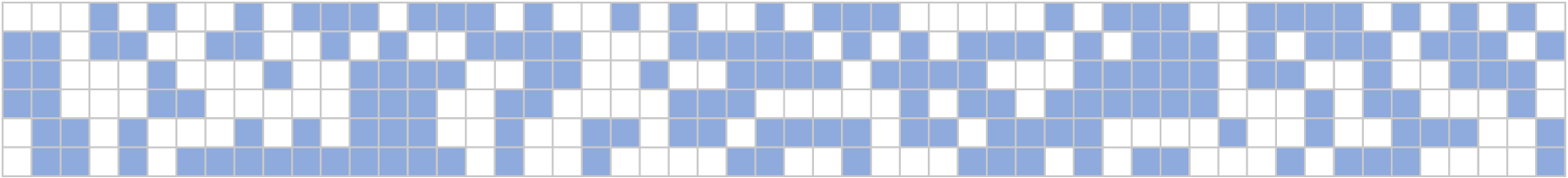}}\\
\subfloat[MODNF (ACC=0.5116)]{\label{fig6:h}\includegraphics[width=0.9\linewidth]{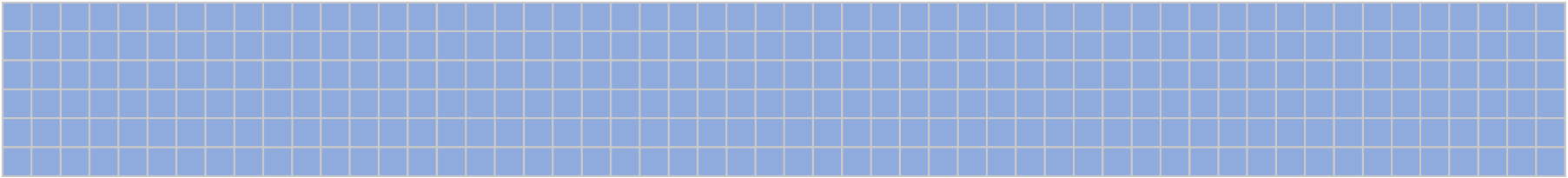}}\\
\subfloat[MODNP (ACC=0.3023)]{\label{fig6:i}\includegraphics[width=0.9\linewidth]{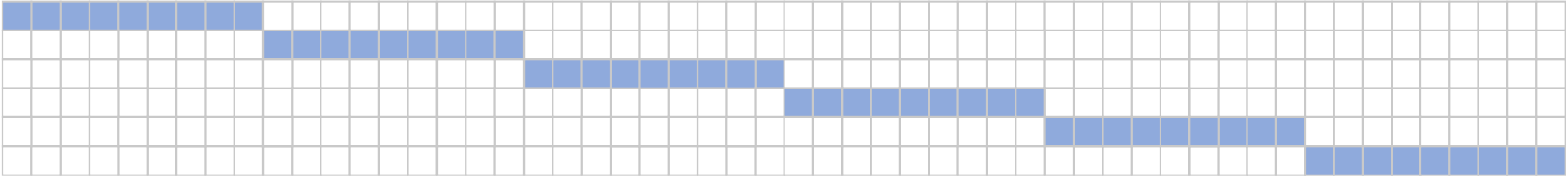}}\\
\caption{Visualization of the preference filter $\boldsymbol{P}$ (transposed) learned with all three models on binary/multi-class classification tasks. Blue and white boxes refer to the 1- and 0-valued elements in $\boldsymbol{P}$. (a)$\sim$(c), (d)$\sim$(f), and (g)$\sim$(i) represent the results obtained on Caesarian, Iris, and Glass datasets, respectively. }
\label{fig:6}
\end{figure}


\section{Conclusion}\label{con}
We have proposed a novel multi-output dendritic model (MODN) to provide a better solution for multi-output tasks. It works by introducing a preference filter into the soma layer, such that the beneficial dendrites can be adaptively selected to relate to each output. Moreover, MODN is a unified formulation that naturally involves a dendritic neuron model by customizing the filtering matrix. We also introduced a telodendron layer to show awareness of the real nervous architecture. To explore the optimization for MODN, we widely investigated the performance of seven heuristic algorithms and one gradient-based method and propose a 2-step training policy to specifically learn the filtering matrix. Experiments on 11 datasets for both binary and multi-out classification tasks generally demonstrated the effectiveness of our model in terms of accuracy, convergence, and generality. 

Because MODN is constructed with a single shallow neuron, it cannot effectively handle large-scale datasets with high dimensionality in the current form. However, simply deepening the MODN would also require a sufficiently powerful optimization strategy for training. In the future, we would like to devise a deep MODN with an appropriate optimizer to better address real-world tasks.


\bibliographystyle{IEEEtran}
\bibliography{reference}

\end{document}